\documentclass{article}

\usepackage{pdfcomment}

    \usepackage[preprint]{neurips_2022}

\usepackage[utf8]{inputenc} 
\usepackage[T1]{fontenc}    
\usepackage{hyperref, mathtools}       
\usepackage{url}            
\usepackage{booktabs}       
\usepackage{amsfonts}       
\usepackage{nicefrac}       
\usepackage{microtype}      
\usepackage{xcolor}         

\title{
A
Kernelised
Stein Statistic for 
Assessing \\
Implicit 
Generative Models
}

\author{%
  Wenkai Xu\\      Department of Statistics \\University of Oxford \\
  \texttt{wenkai.xu@stats.ox.ac.uk} \\
  \And
  Gesine Reinert \\      Department of Statistics \\University of Oxford \\
  \texttt{reinert@stats.ox.ac.uk} \\
}

\usepackage{graphicx}
\usepackage{subfigure}
\usepackage{booktabs} 
\usepackage{amsmath}
\usepackage{amsfonts}
\usepackage{amsthm}
\usepackage{tabularx}
\usepackage{multirow}
\usepackage{soul}
\usepackage{algorithm}
\usepackage{algorithmic}
\usepackage[algo2e]{algorithm2e}
\usepackage{bbm}
\usepackage{comment}
\usepackage[utf8]{inputenc}

\usepackage{tikz}
\usepackage[english]{babel}
\usepackage{bbold}

\usepackage{amsmath}
\usepackage{amssymb}
\usepackage{mathtools}
\usepackage{amsthm}

\usepackage[capitalize,noabbrev]{cleveref}

\theoremstyle{plain}
\newtheorem{theorem}{Theorem}[section]
\newtheorem{proposition}[theorem]{Proposition}
\newtheorem{lemma}[theorem]{Lemma}

\theoremstyle{definition}

\newtheorem{example}[theorem]{Example}
\theoremstyle{remark}
\newtheorem{remark}[theorem]{Remark}

\usepackage{todonotes}

\usepackage[utf8]{inputenc}
\usepackage{amsfonts}
\usepackage[english]{babel}
\usepackage{amsthm}
\usepackage{algorithm}
\usepackage{algorithmic}
\usepackage{bbold}
\usepackage{dsfont}

\usepackage{wrapfig}
\usepackage{enumerate}

\newcommand{\R}{\mathbb{R}} 
\renewcommand{\H}{\mathcal{H}} 
\newcommand{\X}{\mathcal{X}} 
\newcommand{\x}{\boldsymbol{x}} 
\newcommand{\B}{\mathcal{B}} 
\newcommand{\A}{\mathcal{A}} 
\renewcommand{\P}{\mathbb{P}} 
\newcommand{\E}{\mathbb{E}} 
\newcommand{\T}{\mathcal{A}}

\newcommand{\f}{\mathbf{f}}
\newcommand{\g}{\mathbf{g}}

\newcommand{\s}{\mathbf{s}}

\newcommand{\grad}{\nabla} 

\newcommand{\z}{{\boldsymbol{z}}}

\newcommand{\bX}{{\boldsymbol{X}}}

\newcommand{\y}{{\boldsymbol{y}}}

\newcommand{\KSD}{{\operatorname{KSD}}}
\newcommand{\NPKSD}{{\operatorname{NP-KSD}}}
\newcommand{\MMD}{{\operatorname{MMD}}}

\newcommand{\wk}[1]{\textcolor{black}{#1}}
\newcommand{\gr}[1]{\textcolor{black}{#1}}

\begin{document}

\maketitle

\begin{abstract}
{Synthetic data generation has become a key ingredient for training machine learning procedures,}
addressing tasks such as data augmentation, analysing privacy-sensitive data, 
or
visualising representative samples.
{Assessing the quality of such synthetic data generators hence has to be addressed. As} 
(deep) generative models {for synthetic data} often do not admit explicit probability distributions, {classical}
statistical procedures for assessing model goodness-of-fit
may not {be applicable}.
In this paper, we propose a principled procedure to assess the quality of {a synthetic} data generator. {The procedure is a kernelised Stein discrepancy (KSD)-type test which is based on a non-parametric Stein operator for the synthetic data generator of interest. This operator is estimated from samples which are obtained from the synthetic data generator and hence can be applied even when the model is only implicit.} 
{In contrast to classical testing, the sample size from the synthetic data generator can be as large as desired, while the size of the observed data which the generator aims to emulate is fixed.} 
Experimental results on synthetic distributions and trained generative models on {synthetic and} real datasets illustrate {that the method shows} improved power performance compared to existing approaches. 
\end{abstract}



\section{Introduction}\label{sec:intro}
Synthetic data capturing main features of the original dataset are of particular interest for machine learning {methods}. 
The use of original dataset for machine learning tasks can be problematic or even prohibitive in certain scenarios, e.g. 
under authority regularisation
on privacy-sensitive information, training models
on small-sample dataset, 
{or} calibrating models with imbalanced groups.
High quality synthetic data generation procedures surpass some of these challenges by creating de-identified data 
to preserve privacy and {to} augment small or imbalance datasets.
Training deep generative models has been widely studied in the recent years \citep{kingma2013auto, radford2015unsupervised, song2021train} and
methods such as those based on Generative Adversarial Networks (GANs) \citep{goodfellow2014generative} provide
powerful approaches that learn to generate synthetic data {which}  resemble the original data distributions. However, {these  deep generative models  usually do not provide theoretical guarantees on the} \emph{goodness-of-fit} to the original data 
\citep{creswell2018generative}.

To the best of our knowledge, existing mainstream development{s} for deep generative models \citep{song2020improved, li2017mmd} do not provide {a} systematic approach to assess the quality of the synthetic samples. Instead, heuristic methods {are applied}, e.g. for image data, the quality of samples are generally decided via visual comparisons. The training quality has been studied relying largely on the specific choice of training loss, which does not directly {translate into a measure of sample quality; in the case of the log-likelihood} \citep{theis2015note}.
Common quality assessment measures for implicit generative models,
on images for example,
include
Inception Scores (IS) \citep{salimans2016improved} and
Fréchet Inception Distance (FID) \citep{heusel2017gans}, which are motivated by human inception systems in the visual cortex and pooling \citep{wang2004image}.
\citet{binkowski2018demystifying} pointed out issues for IS and FID and developed {the} Kernel Inception Distance (KID) for more general datasets. 
{Although} 
these scores {can be used for} 
for comparisons, {they do not provide a statistical significance test which would assess}
whether a {deemed} \textit{good} generative model is 
``\textit{good enough}''. 
{A key stumbling block is that the distribution from which a synthetic method generates samples is not available; one only ever observes samples from it.} 

{For models in which the density is known explicitly, at least up to a normalising constant, some assessment methods are available.}  
\cite{gorham2017measuring} 
proposed to assess sample quality using discrepancy measures 
called {\it kernelised Stein discrepancy }(KSD). 
{\citet{schrab2022ksd} assesses} the quality
{of} generative models on {the}  MNIST image dataset {from} \citet{lecun1995learning} using {an} aggregated kernel Stein discrepancy (KSDAgg) test{;} 
{still an} explicit density is 
required.
{The only available implicit goodness-of-fit test, AgraSSt \citep{xu2022agrasst}, applies only to generators of {finite} graphs; 
it is also of KSD form and makes extensive use of the discrete and finite nature of the problem. To date,} 
quality assessment procedures {of}
\textit{implicit} deep generative models {for continuous data} remains {unresolved}. {This paper provides a solution of this problem.} 

{The underlying idea can be sketched as follows.} 
{Traditionally, given a set of $n$ observations, {each in $\R^m$}, one would estimate the distribution of these observations from the data and then check whether the synthetic data can be viewed as coming from the data distribution. Here {instead} we 
characterise the distribution which is generated possibly implicitly from the synthetic data generator, and then test whether the observed data can be viewed as coming from the synthetic data distribution. The advantage of this approach is that while the observed sample size $n$ may be fairly small, the synthetic data distribution can be estimated to any desirable level of accuracy by generating a large number of samples. {Similarly to the works mentioned in the previous paragraph 
for goodness-of-fit tests,} 
we
use a KSD approach, based on a Stein operator which characterises the synthetic data distribution. As the synthetic data generator is usually implicit, this Stein operator is not available. We show however that it can be estimated from synthetic data samples to any desired level of accuracy. }

\paragraph{Our contributions}
{W}e {introduce} a method 
to assess 
(deep) generative models, {which are often \textit{black-box}} 
{approaches},  when the underlying probability distribution is continuous, usually in high-dimensions.
To this purpose, we develop a non-parametric Stein operator and the corresponding non-parametric kernel Stein discrepancies (NP-KSD), based on estimating conditional score functions. Moreover, we {give theoretical} 
guarantees for NP-KSD.
%

{This paper is structured as follows.} 
We start {with} a review {of} Stein's method and KSD {goodness-of-fit} tests 
for explicit models in \cref{sec:stein_review} before we introduce the NP-KSD in \cref{sec:np-ksd} {and}
analyse the model assessment procedures. {We}
show {results of} experiments 
in \cref{sec:exp} and conclude with future directions in \cref{sec:conclusion}. {Theoretical underpinnings, and additional results are provided in the supplementary material.
}
The code 
is available at \url{https://github.com/wenkaixl/npksd.git}.

\section{Stein's method and kernel Stein discrepancy tests}\label{sec:stein_review}

\paragraph{{{Stein identities, equations, and operators}}} Stein's method \citep{stein1972bound} provides an elegant tool to characterise distributions via \emph{Stein operators}, which {can be used to assess distances between probability distributions} 
\citep{barbour2005introduction, barbour2005multivariate, barbour2018multivariate}.
Given a distribution $q$, an operator
$\A_q$ is called a Stein operator w.r.t. $q$ {and {\it Stein class} $\mathcal{F}$} if the following {Stein} identity holds for any \emph{test function} $f{\in \mathcal{F}}$: 
$    {\E}_q [\mathcal{A}_q {f}]=0.$
{For a test function $h$ one then aims to find a function $f = f_h {\in \mathcal{F}}$ which solves the {\it Stein equation}}
\begin{equation}\label{eq:stein_equation}
    \T_q f (\x) = h(\x) - \E_q [h(\x)].
\end{equation}
{Then for any distribution $p$, taking expectations $\E_p$ in Eq.\,\ref{eq:stein_equation} assesses the distance $ | \E_p h - \E_q h|$ through $| \E_p \T_q f|$, an expression in which randomness enters only through the distribution $p$.} 

When the density function {$q$}  is given explicitly, {with smooth support $\Omega_q {\subset \R^m}$,  is differentiable and vanishes at the boundary of $\Omega_q$,} a common choice of Stein operator in the literature utilises the score-function, {see for example \cite{mijoule2021stein}}.  {The gradient operator is denoted by}
$\nabla$
and 
taken to be a column vector. 
%
The \emph{score function} of $q$ is defined as $\s_q = \grad \log q = \frac{\nabla q}{q}$ 
(with the convention that $
\s_q \equiv 0$ outside of $\Omega_q$).
Let $\f = (f_1,\dots, f_{{m}})^\top$ where $f_i: \R^{{m}} \to \R, \forall i,$ are differentiable.
The \emph{score-Stein operator}\footnote{also referred to as Langevin Stein operator \citep{barp2019minimum}.} is the vector-valued operator 
acting on (vector-valued) function $\f$,
\begin{align}
\mathcal{A}_q \f(\x) &= \f(\x)^{\top} \nabla \log q(\x)+ \nabla \cdot \f(\x),
\label{eq:steinRd}
\end{align}
and the {Stein} identity $\E_q [\A_q f] = 0$ {holds for functions $f$ which belong to the so-called {\it canonical Stein class}  defined in \citet{mijoule2021stein}, Definition 3.2.}  
{As it requires knowledge of the density $q$}
only via its score function, this Stein operator is particularly useful for unnormalised densities \citep{hyvarinen2005estimation}, {appearing} e.g. {in} energy based models (EBM) \citep{lecun2006tutorial}. 
%

\paragraph{Kernel Stein discrepancy}
Stein operators can be used to assess discrepancies between two probability distributions;
%
{t}he Stein discrepancy between probability distribution $p$ and $q$ (w.r.t. 
class ${\mathcal B} \subset {\mathcal F}$) 
is defined as \citep{gorham2015measuring}
\begin{equation}
\operatorname{SD}(p\|q, {{\mathcal B }}) =\sup_{f \in \mathcal B}  \{|   \mathbb{E}_{p}[{\T}_q f] -\underset{=0}{\underbrace{\mathbb{E}_{p}[\mathcal{A}_p f]}} |  \} = \sup_{f \in \mathcal B} {|}   \mathbb{E}_{p}[{\T}_q f] {|}.
\label{eq:sd}
\end{equation} 
{As} the $\sup f$ {over a} general class $\mathcal B$ can be difficult to compute, 
{taking $\mathcal B$ as the unit ball of a} reproducing kernel Hilbert space (RKHS) has been considered, {resulting in the} 
{\it kernel Stein discrepancy }(KSD) 
defined as \citep{gorham2017measuring} 
\begin{equation}
\operatorname{KSD}(p\|q, {{\mathcal H }}) =\sup_{f \in \B_1(\mathcal H)} {|}  \mathbb{E}_{p}[{\T}_q f]{|} . 
\label{eq:ksd}
\end{equation}
Denoting {by} $k$  the reproducing kernel associate{d} with {the} RKHS $\H$ over a set $\mathcal X$, {the reproducing property  ensures that} $\forall f \in \H$, $f(\x) = \langle f, k(\x,\cdot) \rangle_{\H}, \forall \x \in \X$.
Algebraic manipulations yield
\begin{align}
\mathrm{KSD}^2(q\|p) = {\E}_{\x,\tilde{\x} \sim p} [u_q(\x,\tilde{\x})], 
\label{eq:KSDequiv}
\end{align}
where $u_q (\x,\tilde{\x})= \langle \T_q k(\x,\cdot), \T_q k(\tilde \x,\cdot)\rangle_{\H}$, which takes the exact $\sup$ without approximation and
does not involve the (sample) distribution $p$. 
{Then,  KSD$^2$ can be estimated through empirical means}, 
over samples from $p$,
e.g. V-statistic \citep{van2000asymptotic} and U-statistics \citep{lee90} estimates are
\begin{equation}\label{eq:u-stats}
{{\operatorname{KSD}}_v^2}(q\|p)=\frac{1}{m^2}\sum_{i \gr{,} j}u_{q}(\x_{i},\x_{j}),
\qquad
{{\operatorname{KSD}}_u^2}(q\|p)=\frac{1}{m(m-1)}\sum_{i\neq j}u_{q}(\x_{i}\,\x_{j}).
\end{equation}
KSD 
has been studied as discrepancy measure 
between distributions
{for} testing model goodness-of-fit \citep{chwialkowski2016kernel, liu2016kernelized}.

%
\paragraph{KSD testing procedure}
Suppose we have observed samples $\x_1,\dots,\x_n$ from the \emph{unknown} distribution $p$. 
To test the null hypothesis $\mathrm{H_0}: p=q$ against the (broad class of) alternative hypothesis $\mathrm{H_1}: p\neq q$,
KSD can be empirically estimated via Eq.\,\ref{eq:u-stats}.
The null distribution is {usually} simulated via the wild-bootstrap procedure \citep{chwialkowski2014wild}. Then {if} the empirical quantile, i.e. the proportion of wild bootstrap samples that are larger than ${\operatorname{KSD}_v^2}(q \| p)$, 
is smaller than the pre-defined test level (or significance level) $\alpha$,  the null hypothesis is rejected; otherwise the null hypothesis is not rejected.
In this way, a systematic non-parametric goodness-of-fit testing  procedure  is obtained, which is applicable to unnormalised models. 
%


\section{Non-Parametric kernel Stein discrepancies}\label{sec:np-ksd}

The construction of {a} KSD relies on the knowledge of the density model, up to normalisation. However, for deep generative models where the density function is not  explicitly known, the computation for Stein operator in Eq.\,\ref{eq:steinRd}, which is based on an explicit parametric density, is no longer feasible.

{While in principle one could estimate the multivariate density function from synthetic data, density estimation in high dimensions is known to be problematic, see for example \cite{scott2005multidimensional}. Instead, Stein's method allows to use a two-step approach: {For data in $\R^m$,} we first pick a coordinate $i \in [m]:= \{1, \ldots, m\}$, and then we characterize the uni-variate conditional distribution of that coordinate, given the values of the other coordinates. Using score Stein operators from \cite{ley2017stein}, this approach only requires knowledge or estimation of uni-variate conditional score functions.
}  

{
{We} denote observed data $\z_1, \ldots, \z_n$ with $\z_i = (z_i^{(1)}, \ldots, z_i^{(m)})^{{\top}} \in \R^m$; and denot{ing}
the generative model as $G$,
we write $\bX \sim G$ to denote a random {$\R^m$-valued element} from the (often only given implicitly) distribution which is underlying $G$.  Using $G$, we generate $N$ samples {denoted by} $\y_1, \ldots, \y_N$.} 
In our case, $n$ is fixed and $n\ll N$, allowing $N \rightarrow \infty$ {in theoretical results}. 
{The kernel of an RKHS is denoted by $k$ and 
is assumed to be bounded.} 
{For $\x \in \R^m$, $x\in\R$ and $g(\x):\mathbb{R}^m \rightarrow\mathbb{R}$,} we write 
$g_{x^{(-i)}}(x): \mathbb{R} \rightarrow\mathbb{R}$ for the uni-variate function which acts only on the coordinate $i$ and fixes the other coordinates {to equal} $x^{(j)}, j \ne i$, so that 
$g_{x^{(-i)}}(x)= g( x^{(1)}, \ldots, x^{(i-1)}, x, x^{(i+1)}, \ldots, x^{(m)}) $.

{For $i \in [m]$ let
$\mathcal{T}^{(i)}$ denote a Stein operator for the conditional distribution $Q^{(i)} = Q^{(i)}_{x^{(-i)}}$ with 
$\E_{Q^{(i)}_{x^{(-i)}}} g_{x^{(-i)}}(x) = \mathbbm{E}[g_{y^{(-i)}}(Y)| Y^{(j)} = y^{(j)}, j \ne i] $.
} 
The {proposed} Stein operator $\mathcal{A} $ acting on functions $g: \mathbb{R}^m \rightarrow\mathbb{R}$ 
underlying the non-parametric Stein operator is 
\begin{equation}\label{summ2}
\mathcal{A} g(x^{{(1)}}, \ldots, x^{{(m)}})
=  \frac1m {{\sum_{i=1}^m \mathcal{T}^{(i)} g_{x^{(-i)}}(x^{(i)})}}.
\end{equation}
We note that for $\bX \sim q$, {the Stein identity} 
$\mathbb{E} \mathcal{A} g(\bX) =0$ {holds} and thus $\mathcal{A} $ is a  Stein operator. The domain of the operator 
will depend on the conditional distribution in question. {Instead of using the weights $w_i = \frac1m$, other positive weights which sum to 1 would be possible, but for simplicity we use equal weights. A 
{more} detailed theoretical
justification {of Eq.\,\ref{summ2}} is given in 
\cref{app:justify}.}

{In what follows we use as Stein operator for a differentiable uni-variate density $q$ the score operator from Eq.\,\ref{eq:steinRd}, given by 
\begin{eqnarray} \label{opform}
\mathcal{T}_q^{(i)} f (x) = f'(x)  + f(x)  \frac{q'(x)}{q(x)} .
\end{eqnarray} } 

{In \cref{prop:equal} of \cref{app:three_approaches} we shall see that the operator in Eq.\,\ref{summ2} equals the score-Stein operator in Eq.\,\ref{eq:steinRd}{; in \cref{app:three_approaches}  an example is also given}. For the development in this paper, Eq.\,\ref{summ2} is more convenient as it relates directly to conditional distributions.} 
{Other} choice{s of Stein operators}
{are} discussed for example in \cite{ley2017stein, mijoule2021stein,xu2021standard}.

\paragraph{Re-sampling Stein operators}

{The Stein operator Eq.\,\ref{summ2} depends on all coordinates $i\ \in [m]$. When $m$ is large we can estimate this operator via re-sampling with replacement, as follows. We draw $B$ samples $\{i_1,\dots, i_B\}$ with replacement from $[m]$ such that}
$\{i_1,\dots, i_B\}\sim \operatorname{Multinom}({B}, \{\frac1m\}_{i\in[m]})$. The 
re-sampled 
Stein operator {acting on $f: \R^m \rightarrow \R$} {is}
\begin{equation}\label{eq:glauber-stein_resample}
     \T^B f ({{{\z}}}) := \frac{1}{B}\sum_{b=1}^B \A^{(i_b)} f({{{\z}}}).
\end{equation}
Then we have 
$ \E \T{^B} f(\bX) =  \frac{1}{B}\sum_{b=1}^B  \E  \A^{(i_b)} f(\bX) = 0.$
So {$\T{^B}$} is again a Stein operator. 

In practice, when $m$ is large, the stochastic operator in Eq.\,\ref{eq:glauber-stein_resample} creates a computationally efficient way for comparing distributions.
{A} similar re-sampling strategy for constructing
stochastic operator{s are} 
considered 
 in the context of Bayesian inference \citep{gorham2020stochastic}, where 
 conditional score function{s, which are given in parametric form,} are re-sampled to derive score-based (or Langevin) Stein operators for posterior distributions.
The conditional distribution has been considered \citep{wang2018stein} and \citep{zhuo2018message} in the context of graphical models 
\citep{liu2016stein}. In graphical models, the conditional distribution
is simplified to conditioning on the Markov blanket \citep{wang2018stein}, which is a subset of the full coordinate; however, no random re-sampling {is used}.
{Conditional distributions also apply in message passing, but there, the sequence of updates is ordered.} 


\paragraph{Estimating Stein operators via score matching}

{
Usually the score function $q'/q$ in Eq.\,\ref{opform} is not available but needs to be estimated. An efficient way of estimating the score function is through score-matching, see for example \citep{hyvarinen2005estimation, song2021train, wenliang2019learning}.}
Score matching relies on the following score-matching (SM) objective \citep{hyvarinen2005estimation},
\begin{equation}\label{eq:sm_objective}
    J(p\|q) = \E_p \left[\left\|\nabla \log p({\x}) - \nabla \log q({\x})\right\|^2\right],
\end{equation}
which is particularly useful for unnormalised models {such as EBMs}. Additional details are included in \cref{app:sm_obj}.
{Often score matching estimators can be shown to be consistent, see for example \cite{song2020sliced}.
}
\cref{prop:Stein operators}, proven in 
\cref{app:proofs}, gives theoretical guarantees for the consistency of a general form of Stein operator estimation, 
as follows.

\begin{proposition}
\label{prop:Stein operators}
Suppose that  for $ i\in [m]$,  $\widehat{s}_{N}^{(i)} $ is a consistent estimator of the {uni-variate} score function  $ s^{(i)}$. Let $\mathcal{T}^{(i)}$ be a Stein operator for the uni-variate differentiable probability distribution ${Q}^{(i)}$ of the generalised density operator form Eq.\,\ref{opform}.
Let 
\begin{eqnarray*} 
{\widehat{\mathcal{T}}}^{(i)}_{N} g (x) = g'(x)+ g(x)  \widehat{s}_{N}^{(i)} \quad \quad \operatorname{ and } \quad \quad  \widehat{\mathcal{A}} g  = 
\frac{1}{m}\sum_i
{\widehat{\mathcal{T}}}^{(i)}_{N} g_{ x^{(-i)}}. 
\end{eqnarray*} 
Then 
${\widehat{\mathcal{T}}}^{(i)}_{N} $ is a consistent estimator for ${{\mathcal{T}}}^{(i)}$, {and} 
$\widehat{\mathcal{A}} $
is a consistent estimator of $\mathcal{A}.$
\end{proposition} 

\paragraph{Non-parametric Stein operators with summary statistics}

In practice, the data ${y}^{(-i)}\in \R^{m-1}$ can 
be high dimensional,  e.g. image pixels, and the observations can be sparse. Thus, estimation of the conditional distribution can be unstable or exponentially large sample size is required. 
Inspired {by} \cite{xu2021stein} and \cite{xu2022agrasst}, {we use low-dimensional measurable non-trivial summary statistics $t$ and the conditional distribution of the data given $t$ as new target distributions. {Heuristically,} 
if two distributions match, then so do their conditional distributions.
{Thus, the  conditional distribution} $Q{^{(i)}}(A)$ is
replaced by 
$Q_t^{(i)}(A) = \P(X^{(i)} \in A |t(x^{(-i)} ))$.
Setting $t(x^{(-i)}) = x^{(-i)}$ replicates the actual conditional distribution. {We denote the uni-variate score function of $q_t(x | t(x^{(-i)})) $ by ${s_{t}^{(i)}(x | t(x^{(-i)}}))$, 
or by $s_{t}^{(i)}(x)$ when the context is clear}.
The summary statistics $t(x^{(-i)} )$ can be uni-variate or multi-variate, {and {they may attempt} to capture} 
useful distributional features. %
{Here} we consider uni-variate summary statistics {such as}
the sample mean.

\begin{algorithm}[t!]
  \caption{Estimating {the} conditional probability 
  via summary statistics
   }
   \label{alg:est_conditional}
\begin{algorithmic}[1]
\renewcommand{\algorithmicrequire}{\textbf{Input:}}
\renewcommand{\algorithmicensure}{\textbf{Procedure:}}
\REQUIRE
Generator $G$; summary statistics $t(\cdot)$; {number of samples} $N$ from $G$; re-sample size $B$
\ENSURE~~\\
\STATE Generate samples $\{{\y}_1,\dots,{\y}_N\}$ from $G$.
\STATE Generate coordinate index sample $\{i_1,\dots,i_B\}$
\STATE For ${i_b\in[m], l\in[N]}$, estimate $q(z^{(i_b)}| t(z^{-i_b})$ from samples
$\{{y}^{(i_b)}_l, t({y_l}^{-i_b})\}_{l \in [N]}$ via {the} score-matching objective in Eq.\,\ref{eq:sm_objective}.
\renewcommand{\algorithmicensure}{\textbf{Output:}}
\ENSURE
$\widehat s_{t, N}^{(i)}(z^{(i)} | t(z^{(-i)})), \forall i\in [m]$.
\end{algorithmic}
\end{algorithm}

The non-parametric Stein operator enables the construction of Stein-based statistics based on 
Eq.\,\ref{summ2}
with estimated score functions $\widehat s_{t,N}^{(i)}$ using generated samples from the model $G$, as shown in \cref{alg:est_conditional}. 
The re-sampled non-parametric Stein operator {is} 
$$\widehat{\mathcal{A}^B_{t,N}} g  = \frac{1}{B} \sum_b {\widehat{\mathcal{T}}}^{(i_b)}_{t,N} g_{ x^{(-i_b)}} { =\frac{1}{B} \sum_b \left( g_{ x^{(-i_b)}}' + g_{ x^{(-i_b)}} \widehat{s}_{t,N}^{(i)} \right)}.
$$ 

\paragraph{Non-parametric kernel Stein discrepancy}

With the well-defined non-parametric Stein operator, we define the corresponding non-parametric Stein discrepancy (NP-KSD)
using {the} Stein operator in Eq.\,\ref{eq:glauber-stein_resample},  the Stein discrepancy notion in Eq.\,\ref{eq:sd} and
choosing {as set of} test function{s the unit ball of the RKHS}  within unit ball RKHS. 
Similar{ly} to Eq.\,\ref{eq:ksd}, {we define the} NP-KSD with summary statistic $t$ 
as
\begin{equation}\label{eq:np-ksd}
    \operatorname{NP-KSD}_t(G\|p) = \sup_{f \in \B_1(\H)} \E_p[\widehat \A_{t,N}^B f] .
\end{equation}
{A s}imilar quadratic form as {in}  Eq.\,\ref{eq:KSDequiv} applies {to give} 
\begin{align}
{\NPKSD}_t^2(G\|p) = {\E}_{{\x},{\tilde \x} \sim p} [\widehat u^B_{t,N}({\x}, {\tilde{\x}})], \label{eq:NPKSDequiv}
\end{align}
where $\widehat u^B_{t,N} ({\x},{\tilde \x})=
\langle \widehat \A^B_{t,N} k(\x,\cdot), \widehat \A^B_{t,N} k(\tilde \x,\cdot)\rangle_{\H}$. The empirical estimate  {is}
\begin{align}
\widehat{\NPKSD}_t^2(G\|p) = \frac{1}{n^2} \sum_{i,j \in [n]} [\widehat u^B_{t,N}(\z_i, \z_j)], \label{eq:NPKSDempirical}
\end{align}
{where  $\mathbb{S}=\{\z_1, \dots, \z_n\} \sim p$.}
}
{Thus,} 
NP-KSD allows the 
computation
between a set of samples and a generative model, enabling the quality assessment {of} synthetic data generators even for implicit models.

The relationship between NP-KSD and KSD is clarified in the following result; {we use the notation 
$ \hat{\bf s}_{t,N}= (\hat{s}_{t,N} (x^{{(i)}}), i \in [m]) $.
Here we set 
\begin{equation}
    \label{eq:ksdt} 
{\KSD}_t^2(q_t\|p) 
= {\E}_{\x,\tilde{\x} \sim p} [\langle \A_t k(\x,\cdot),  \A_t k(\tilde \x,\cdot)\rangle_{\H} 
\quad \mbox{ with }  \quad 
\A_t g(\x) 
:= \frac1m \sum_{i=1}^m \mathcal{T}_{q_t}^{(i)} g_{x^{(-i)}}(
x^{(i)}
)
\end{equation}
as in Eq.\,\ref{summ2}, and following Eq.\,\ref{opform},
$
\mathcal{T}_{q_t}^{(i)} g_{x^{(-i)}}(x)
= g_{x^{(-i)}}'(x) + 
g_{x^{(-i)}} (x) s_{t}^{(i)} (x | t(x^{(-i)} )).
$}
More details about the interpretation of this quantity are given in App.\,\ref{app:asymp}.

{\begin{theorem}\label{th:KSDconv} 
Assume that the score function estimator vector $ \hat{{\bf{s}}}_{t,N}{= (\hat{{s}}_{t,N}^{(i)}, i=1, \dots, m )^\top}$ is asymptotically normal with mean $0$ and covariance  matrix $N^{-1} \Sigma_{s}$. Then 
${\NPKSD}_t^2(G\|p)$ converges in probability to ${\KSD}_t^2(q_t\|p)$ at rate at least $\min(B^{-\frac12}, N^{-\frac12})$. 
\end{theorem}
The proof of Theorem \ref{th:KSDconv}, which is found in App.\,\ref{app:proofs}, also shows that the distribution  ${\NPKSD}_t^2(G\|p) -{\KSD}_t^2(q_t\|p)$  involves mixture of normal variables. The assumption of asymptotic normality for score matching estimators is often satisfied, see for example \cite{song2020sliced}. 
}

\paragraph{Model assessment with NP-KSD}

Given an implicit generative model $G$
and a set of observed samples
$\mathbb S = \{{\z}_1,\ldots,{\z}_n\}$, we aim to test the null hypothesis $H_0:\mathbb S \sim G$ versus the alternative $H_1:\mathbb S \not\sim G$. {This test} 
assume{s that}  samples generated from $G$ follows some (unknown) distribution $q$ and $\mathbb{S}$ are generated according to some (unknown) distribution $p$. The null hypothesis is $H_0: p = q$ while the alternative is $H_1: p \neq q$.
We note that the observed sample size $n$ is fixed.

\paragraph{NP-KSD testing procedures}
NP-KSD can be {applied} for testing the above hypothesis {using the}
testing procedure
outlined in \cref{alg:NP-KSD}. {In contrast} 
to the KSD testing procedure in \cref{sec:stein_review}, {the} NP-KSD test in \cref{alg:NP-KSD} is a Monte Carlo based test \citep{xu2021stein,xu2022agrasst,schrab2022ksd} {for which} 
the null distribution is approximated via samples generated from $G$ instead of the wild bootstrap procedure \citep{chwialkowski2014wild}.
{The reasons for} 
employing the Monte Carlo testing strategy instead of the wild-bootstrap {are} 1). {T}he non-parametric Stein operator depends on {the} random function $\widehat s_t$ so that {classical results for} 
V-statistics convergence {which assume that the sole source of randomness is the bootstrap}  may not apply\footnote{{A} KSD with random Stein kernel has been briefly discussed in \cite{fernandez2020kernel} when the $h_q$ function requires estimation from relevant survival functions.};
2). {While} the wild-bootstrap is asymptotically consistent {as observed sample size $n \rightarrow \infty$}, {it}  
may not necessarily control the type-I error 
{in a} non-asymptotic regime {where $n$ is fixed}. {More details can be found in \cref{app:mmd}.}



{Here we note that any test which is based on the summary statistic $t$ will only be able to test for a distribution up to equivalence of their  distributions with respect to the summary statistic $t$; two distributions $P$ and $Q$ 
 are equivalent w.r.t. the summary statistics $t$ if $P(\bX|t(\bX)) = Q(\bX|t(\bX))$.
Thus the null hypothesis for the NP-KSD test is that the distribution is equivalent to $P$ with respect to $t$. Hence, the null hypothesis specifies the conditional distribution, not the unconditional distribution.} 


\begin{algorithm}[t!]
   \caption{Assessment procedures for implicit generative models}
   \label{alg:NP-KSD}
\begin{algorithmic}[1]
\renewcommand{\algorithmicrequire}{\textbf{Input:}}
\renewcommand{\algorithmicensure}{\textbf{Objective:}}
\REQUIRE
    Observed sample set $\mathbb{S}=\{{\z}_1,\ldots,{\z}_n\}$; 
    generator $G$ and generated sample size $N$; 
    estimation statistics $t$; RKHS kernel $K$; re-sampling size $B$; { bootstrap sample size $b$}; confidence level $\alpha$;
\renewcommand{\algorithmicensure}{\textbf{Procedure:}}
\STATE Estimate $\widehat s(z^{(i)}|t(z^{(-i)}))$ based on Algorithm \ref{alg:est_conditional}.
\STATE Uniformly generate re-sampling index
$\{i_{1},\dots,i_{B}\}$ from $[m]$, with replacement.
\STATE Compute $\tau =\widehat{\operatorname{NP-KSD}}^2(\widehat s_t;\mathbb{S})$ in Eq.\,(\ref{eq:NPKSDempirical}). 
\STATE Simulate $\mathbb{S}_i=\{{\y}'_1,\dots,{\y}'_n\}$ for ${i}\in[b]$ from $G$.
\STATE Compute $\tau_i =\widehat{\operatorname{NP-KSD}}^2(\widehat s_t;\mathbb{S}_i)$ in  again with index re-sampling. 
\STATE {Estimate} {the} {empirical}   {(1- $\alpha$)}
quantile $\gamma_{1-\alpha}$ via $\{\tau_1,\dots,\tau_b \}$.
\renewcommand{\algorithmicrequire}{\textbf{Output:}}
\REQUIRE
Reject the null {hypothesis} if $\tau > \gamma_{1-\alpha}$; otherwise do not reject.
\end{algorithmic}
\end{algorithm}

\paragraph{Related works}
To assess whether an implicit generative models can generate samples that are \emph{significantly} good for the desired data model, several hypothesis testing procedures have been studied.
\cite{jitkrittum2018informative} has proposed kernel-based test statistics, Relative Unbiased Mean Embedding (Rel-UME) test and Relative Finite-Set Stein Discrepancy (Rel-FSSD) test for relative model goodness-of-fit, i.e. whether model S is a better fit than model R. While Rel-UME is applicable for implicit generative models,  Rel-FSSD still requires explicit knowledge of the unnormalised density. The idea for assessing sample qualit{y} for 
implicit generative models is through addressing two-sample problem, where samples generated from the implicit model {are} compared with the observed data.
In this sense, maximum-mean-discrepancy (MMD) may also apply for assessing sample qualities for the implicit models. With efficient choice of (deep) kernel, \cite{liu2020learning} applied MMD tests to assess the distributional difference for image data, e.g. MNIST \citep{lecun1998gradient} v.s. digits image trained via deep convolutional GAN (DCGAN) \citep{radford2015unsupervised}; CIFAR10 \citep{krizhevsky2009learning} v.s. CIFAR10.1 \citep{recht2019imagenet}. However, as the distribution is represented via samples, the two-sample based assessment suffers from limited probabilistic information from the implicit model 
and low estimation accuracy when the sample size for observed data is small.


\section{Experiments}\label{sec:exp}

\subsection{Baseline and competing approaches}

We 
{illustrate} the proposed NP-KSD testing procedure with different choice of summary statistics. {We denote by} 
\textbf{NP-KSD} {the version which uses} the estimation of {the} conditional score, i.e. $t(x^{(-i)}) = x^{(-i)}$; {by}
\textbf{NP-KSD\_mean} {the version which uses} conditioning on the mean statistics, i.e. $t(x^{(-i)}) = \frac{1}{m-1}\sum_{j\neq i} x^{(j)}$; 
{and by}  \textbf{NP-KSD\_G}  {the version which fits} 
a Gaussian model as  conditional density\footnote{\textbf{NP-KSD\_G} for non-Gaussian densities {is} generally mis-specified. We deliberately check this case to {assess}  the robustness of the NP-KSD procedure under model mis-specification.}.

Two-sample testing methods can be useful for model assessment, where the observed sample set is tested against sample set generated from the model. In our setting where $n \ll N$, we consider a consistent non-asymptotic MMD-based test, \textbf{MMDAgg} \citep{schrab2021mmd}, as our competing approach; {see \cref{app:mmd} for more details}.
%
For synthetic distributions where the null models have explicit densities, we include {the} \textbf{KSD} goodness-of-fit testing procedure in \cref{sec:stein_review} as the baseline.
Gaussian kernels are used and the median heuristic \citep{gretton2007kernel} is applied for bandwidth selection. 
{As a caveat, in view of \citep{gorham2015measuring}, 
when the kernel decays more rapidly than the score function grows, then identifiability of $q_t$ through a KSD method may not be guaranteed. {Details while MMD 
is not included in this list are found in \cref{app:mmd}.} 
}




\subsection{Experiments on synthetic distributions
}

    

\begin{figure}[t!]
    \centering
     {\includegraphics[width=0.7\textwidth]{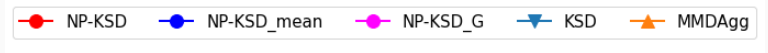}}
     
    \subfigure[GVD: $n=100$
    ]{
    \includegraphics[width=0.239\textwidth]{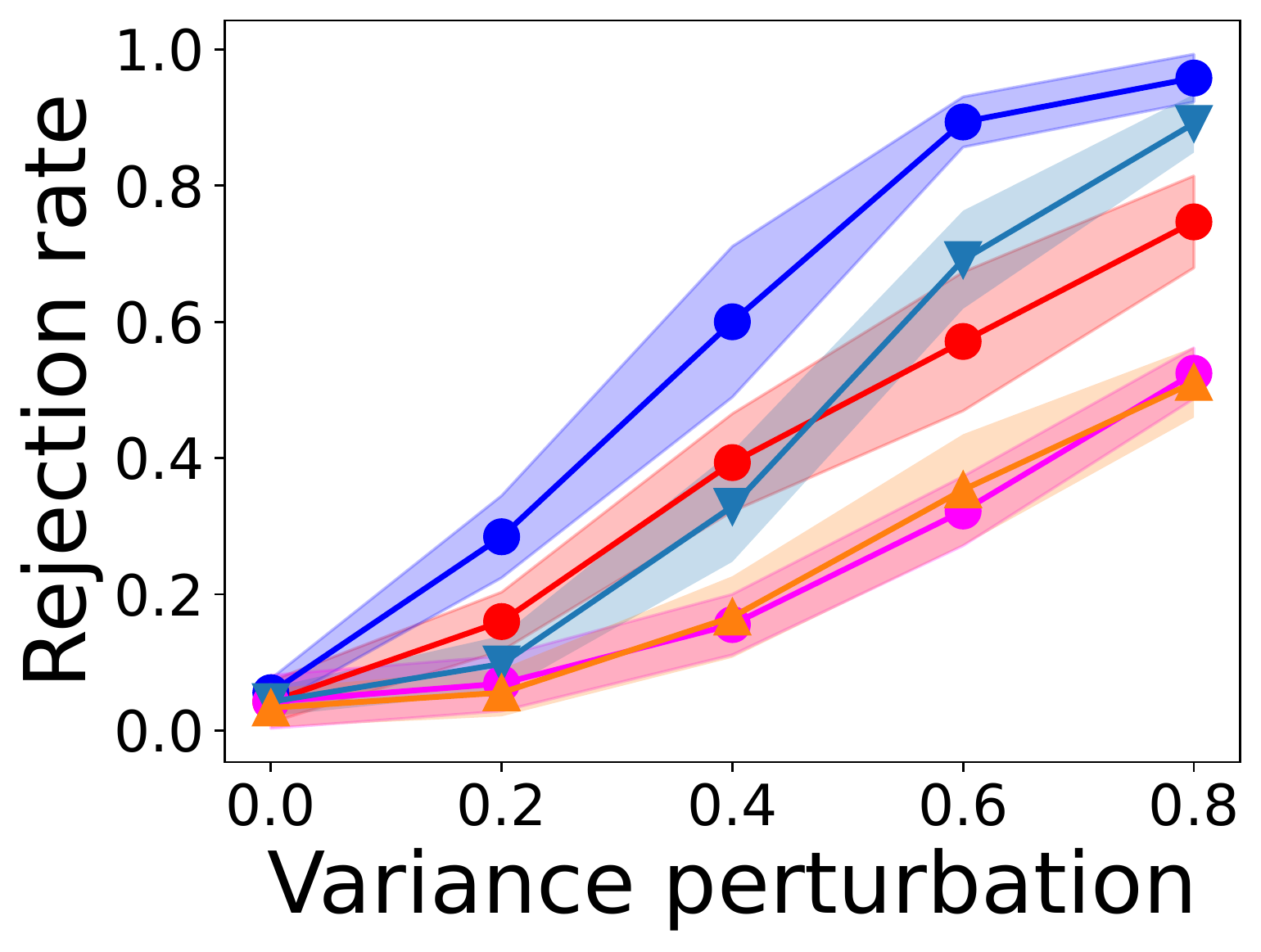}\label{fig:GVD_per}}
        \subfigure[GVD:$\sigma_{per}=0.4$
    ]{
    \includegraphics[width=0.239\textwidth]{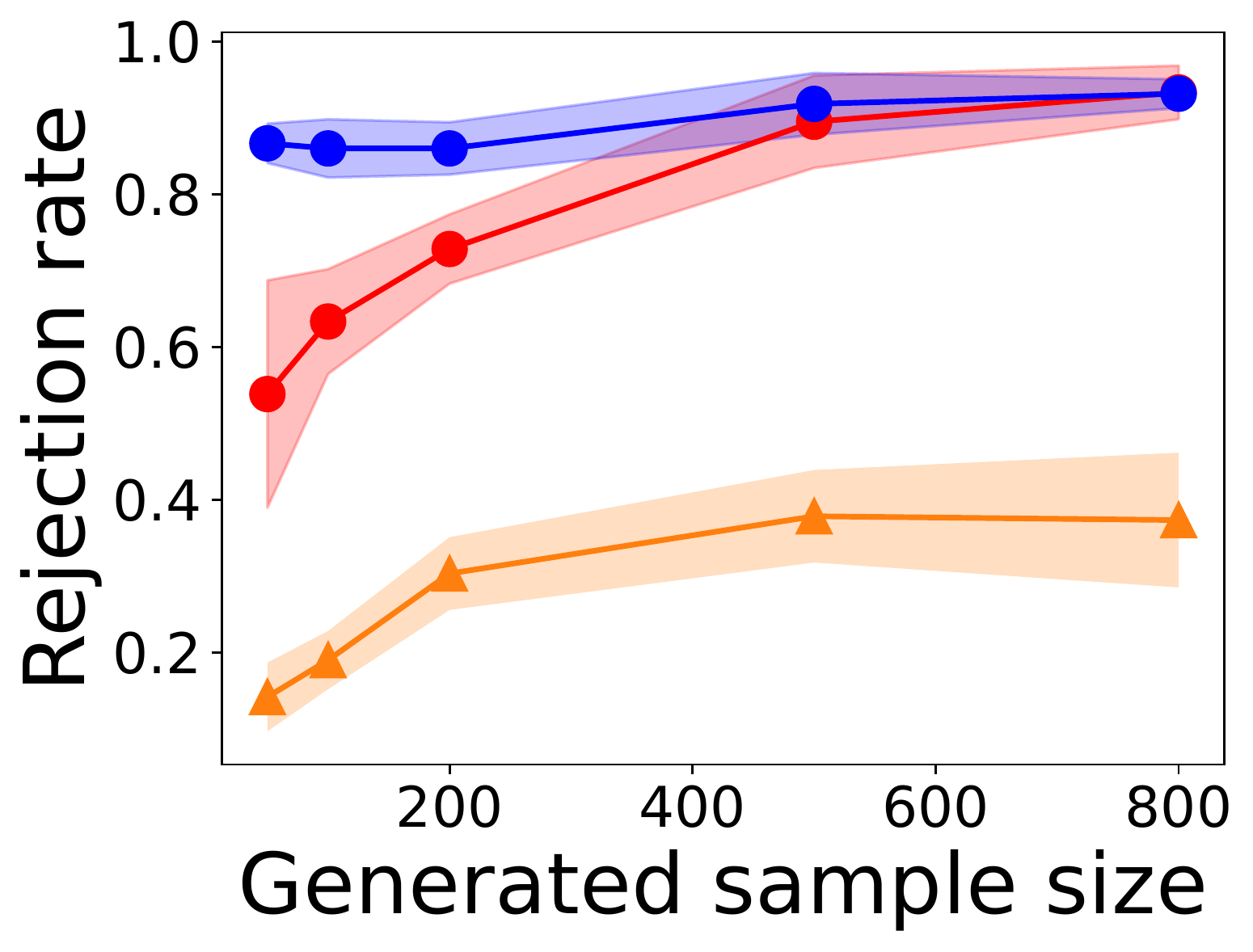}\label{fig:GVD_sample}}    \subfigure[MoG: $n=200$
    ]{
    \includegraphics[width=0.239\textwidth]{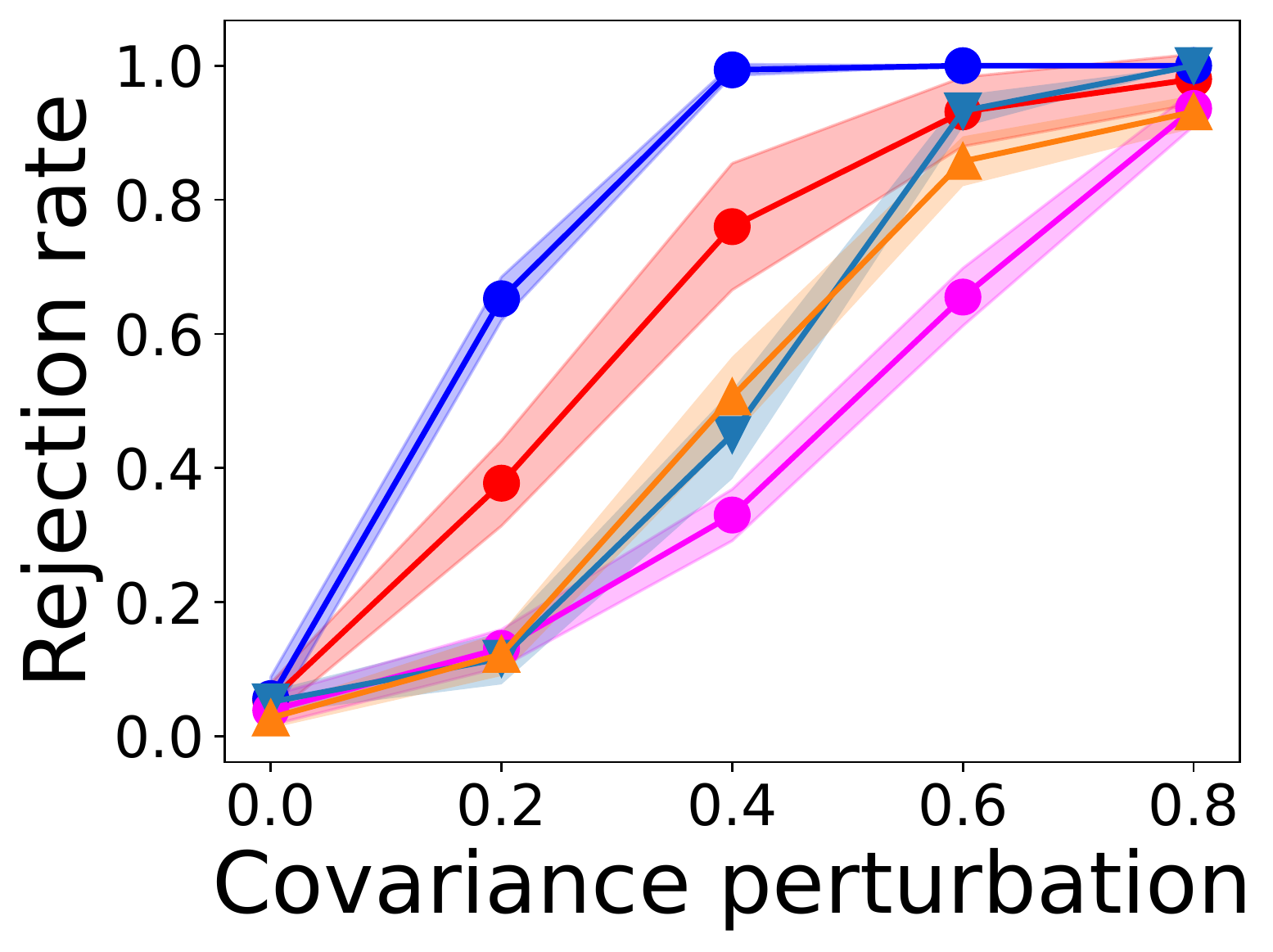}\label{fig:MoG_per}}    \subfigure[MoG: $m=40$
    ]{
    \includegraphics[width=0.239\textwidth]{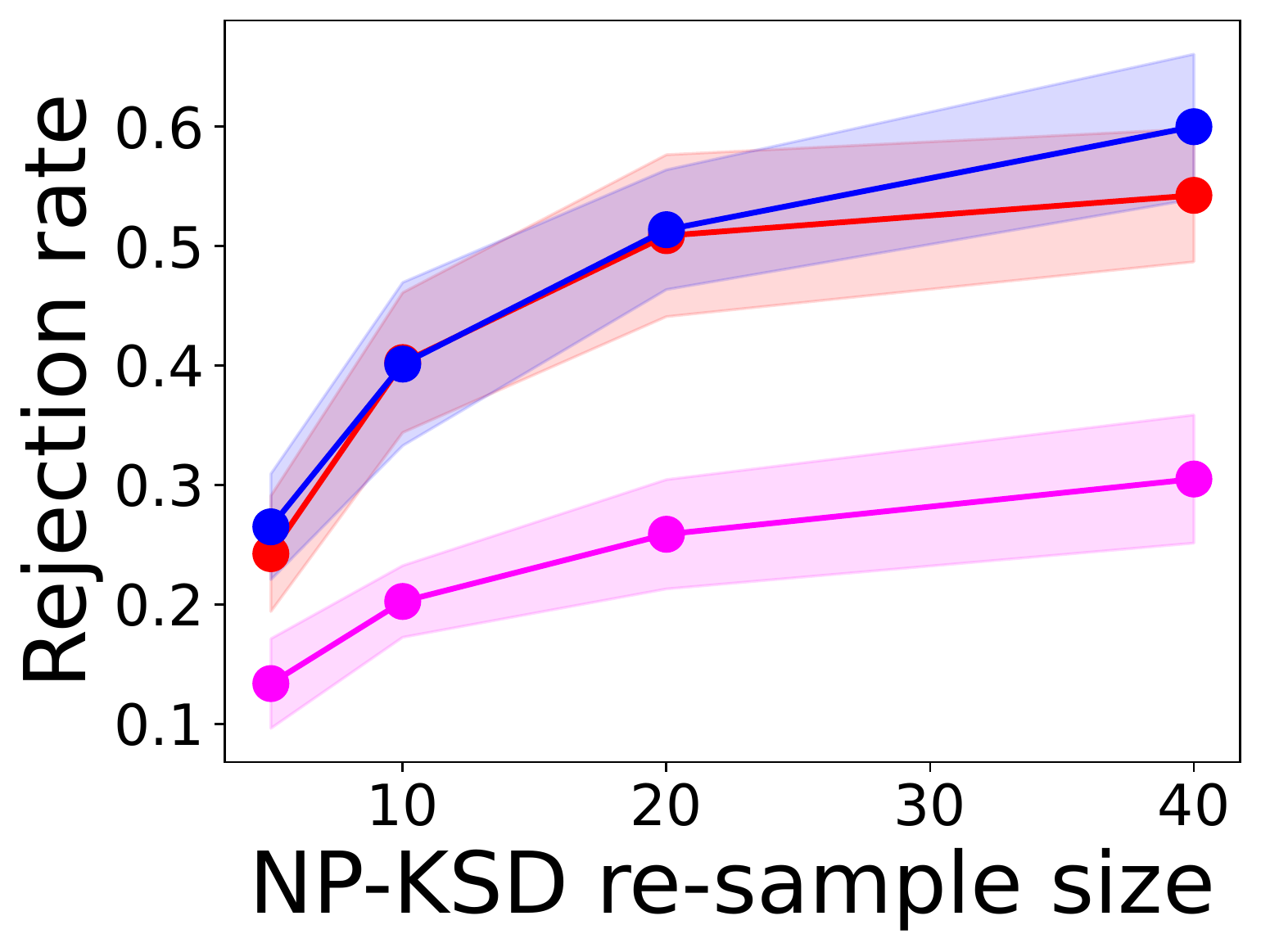}\label{fig:MoG_resample}}
\vspace{-0.1cm}
    \caption{Rejection rates of the synthetic distributions: test level $\alpha = 0.05$; $100$ trials per round of experiment; $10$ rounds of experiment are taken for average and standard deviation; bootstrap sample size $b=500$; 
$m=3$ for     (a) and (b); $m=6$ for (c); $n=100$, $\sigma_{per}=0.5$ for (d).
}
\vspace{-0.35cm}
    \label{fig:exp_synthetic}
\end{figure}


\paragraph{Gaussian Variance Difference (GVD)}
We first consider a standard synthetic setting, 
studied in \citet{jitkrittum2017linear}, {in which}
the null distribution is multivariate Gaussian
with mean zero and identity {covariance} matrix. 
The alternative is set to perturb the 
the diagonal terms of the covariance matrix, {i.e. the variances, all by the same amount.}

The rejection rate against {the} variance{s} perturbation is shown in \cref{fig:GVD_per}. 
From the result, we see that all the tests presented {have} 
controlled type-I error. {For all the tests the power increases with increased perturbation.} 
\textbf{NP-KSD} and \textbf{NP-KSD\_mean} outperform the \textbf{MMDAgg} approach. {Using the}  mean statistics, \textbf{NP-KSD\_mean} is having slightly higher power than \textbf{KSD}. 
The mis-specified \textbf{NP-KSD\_G} has lower power, but is still competitive to \textbf{MMDAgg}.

The test power against the sample size $N$ generated from the null model is shown in \cref{fig:GVD_sample}. The generated samples are used as another sample set for the \textbf{MMDAgg} two-sample procedure, while used for estimating the conditional score for {NP-KSD}-based methods.
As the generated sample size increase{s}, {the power of} \textbf{MMDAgg} {increases more slowly than that of the} 
{NP-KSD}-based methods, {which}
achieve maximum test power in the presented setting. The {NP-KSD}-based tests tend to have lower 
{variability of the} test power, {indicating}  
more reliable testing procedures than 
\textbf{MMDAgg}.
\paragraph{Mixture of Gaussian (MoG)}
{Next}, we consider as a more difficult problem 
{that} the null model is a two-component mixture of {two independent} Gaussians. Both Gaussian components have identity {covariance} matrix.
The alternative is set to perturb the covariance between adjacent {coordinates}.

The rejection rate against {this} perturbation of covariance terms are presented in \cref{fig:MoG_per}. The results show consistent type I error.
The \textbf{NP-KSD} and \textbf{NP-KSD\_mean} {tests} have better test power compared to \textbf{KSD} and \textbf{MMDAgg}, although \textbf{NP-KSD} has slightly higher variance. {Among the $\NPKSD$ tests, the smallest variability is achieved by}  \textbf{NP-KSD\_mean}.
{For the test with $m=40$, we also vary the re-sample size $B$.} 
As shown in \cref{fig:MoG_resample},
while the
{variability} of the average test power also increased slightly. From the result, we also see that
{for $B=20 = m/2$ the test power is already competive compared to $B=40$.}
Additional experimental results including computational runtime and training generative models for synthetic distributions are included in \cref{app:exp}.

\subsection{Applications {to} deep generative models}
For real-world applications, we assess models trained from well-studied generative modelling procedures, including {a} Generative Adversarial Network (\textbf{GAN}) \citep{goodfellow2014generative} with multilayer perceptron (MLP), {a}  Deep Convolutional Generative Adversarial Network (\textbf{DCGAN}) \citep{radford2015unsupervised}, {and a}  Variational Autoencoder (\textbf{VAE}) \citep{kingma2013auto}. 
We also consider {a} Noise Conditional Score Network (\textbf{NCSN}) \citep{song2020improved}, which is a
score-based generative modelling approach, where the score functions are learned \citep{song2019generative} to performed annealed Langevin dynamics for sample generation. 
We also denote \textbf{Real} as the scheme that generates samples randomly from the training data, which essentially acts as a generator of the null distribution.





\paragraph{MNIST Dataset}
{This} dataset contains 
$28\times 28$ grey-scale images of
handwritten digits \citep{lecun1998gradient}\footnote{
\url{https://pytorch.org/vision/main/generated/torchvision.datasets.MNIST.html} }. 
It consist of $60,000$ training samples and $10,000$ test samples.
Deep generative models in \cref{tab:mnist} are trained using the training samples.
We assess the quality of these 
trained generative models by testing against the true observed {MNIST} samples (from the test set).
Samples from both distributions are visually illustrated in \cref{fig:exp_mnist} 
in 
\cref{app:exp}. 

   
\begin{table}[t]
\centering
\begin{tabular}{l|cccc|c}
\toprule
{} & GAN\_MLP & DCGAN & VAE & NCSN & Real \\
\hline

NP-KSD & 1.00 & 0.92 & 1.00 & 1.00 &  0.03 \\
NP-KSD\_m & 1.00 & 1.00 & 1.00 & 1.00 & 0.01 \\
   \hline
   
MMDAgg & 1.00 & 0.73 & 0.93 & 1.00  & 0.06 \\
   \bottomrule
\end{tabular}
  \vspace{0.3cm}
    \caption{Rejection rate for 
    MNIST {generative models.} 
    }\label{tab:mnist}
    \vspace{-0.5cm}
\end{table}
$600$ samples are generated from the generative models and $100$ samples are used for the test; test level $\alpha = 0.05$.
From \cref{tab:mnist}, we see that all the deep generative models have high rejection rate, showing that the trained models are not good enough. Testing with the \textbf{Real} scheme has controlled type-I error. {Thus, NP-KSD detects that the ``real'' data are a true sample set from the underlying dataset.}
%



\paragraph{CIFAR10 Dataset}
{This} dataset contains 
$32\times 32$ RGB coloured images \citep{krizhevsky2009learning}\footnote{
\url{https://pytorch.org/vision/stable/generated/torchvision.datasets.CIFAR10.html} }. 
It consist of $50,000$ training samples and $10,000$ test samples.
Deep generative models in \cref{tab:cifar10} are trained using the training samples and test samples are randomly drawn from the test set.
\begin{table}[t]
\centering
\begin{tabular}{l|ccc|c}
\toprule
{} & 
DCGAN & 
NCSN & CIFAR10.1& Real \\
\hline
NP-KSD & 0.68 & 0.73 & 0.92 &  0.06 \\
NP-KSD\_m & 0.74  & 0.81 & 0.96 & 0.02 \\
   \hline
MMDAgg & 0.48 & 0.57 & 0.83 & 0.07 \\
\bottomrule
\end{tabular}
\vspace{0.3cm}
    \caption{Rejection rate {for}
    CIFAR10 generative models.}\label{tab:cifar10}
    \vspace{-0.5cm}
\end{table}
{Samples are illustrated in Figure \ref{fig:cifar10} in \cref{app:exp}.}
We also compare with {the} CIFAR10.1 dataset\citep{recht2018cifar10.1}\footnote{
\url{https://github.com/modestyachts/CIFAR-10.1/tree/master/datasets}}, which is created to differ from CIFAR10 to investigate generalisation power for training classifiers. 
$800$ samples are generated from the generative models and $200$ samples are used for the test; test level $\alpha = 0.05$. 
\cref{tab:cifar10} {shows}
higher rejection rates {for NP-KSD tests} compared to MMDAgg, echoing the results for synthetic distributions.
The trained \textbf{DCGAN} generates samples with lower rejection rate in the CIFAR10 dataset {than in the CIFAR10.1 dataset}. 
We also see 
that the score-based NCSN has higher rejection rate than the non-score-based DCGAN, despite NP-KSD being a score-based test.  
The distribution difference between CIFAR10 and CIFAR10.1 can be well-distinguished from the tests.
Testing with the \textbf{Real} scheme 
{again} has controlled type-I error.

\section{Conclusion and future directions}\label{sec:conclusion}

Synthetic data are in high demand, for example for training ML procedures; quality is important. Synthetic data which miss important features in the data can lead to erroneous conclusions, which in the case of medical applications could be fatal, and in the case of loan applications for example could be detrimental to personal or business development. NP-KSD provides a method for assessing synthetic data generators which comes with theoretical guarantees. {Our experiments on synthetic data have shown that NP-KSD achieves good test power and controlled {type-I}
error. On real data, 
NP-KSD detects samples from the true dataset.
That none of the classical deep learning methods used in this paper has a satisfactory rejection rate indicates scope for further developments in synthetic data generation.}

Future research will assess alternatives to the computer-intensive Monte Carlo method for estimating the null distribution, {for example adapting wild-bootstrap procedures}. It will explore alternative choices of score estimation as well as of kernel functions.

{Finally, s}ome caution is advised. The choice of summary statistic may have strong influence on the results and a classification based on  NP-KSD  may still miss some features. {Erroneous decisions could be reached when training classifiers. Without scrutiny this could lead to severe consequences for example in health science applications.}  Yet NP-KSD  is an important step towards understanding black-box data generating methods and thus understanding their potential shortcomings. 




\bibliography{main}
\bibliographystyle{plainnat}

\clearpage
\appendix




\section{Justification of the Stein operator}\label{app:justify}
{Here we justify the two-step approach for constructing a Stein operator.}

\subsection{Step 1: A non-parametric Stein operator}\label{sec:glauber_stein}

{
Suppose 
we can estimate the conditional distribution from data. Then we can create a Markov chain with values in $(\mathbb{R}^d)^m$ 
as follows. Starting with $Z_0= \{x_1,\dots,x_m\}$  with $x_i \in  \mathbb{R}^d$ for 
$i=1, \ldots, m$
(often we choose $d=1$),  we pick an index 
$I \in [m]$ at random. If $I=i$ we replace $x_i$ by $X_i'$ drawn from the conditional distribution of $X_i$ given $(X_j: j \ne i)$.
This gives $Z_1= (x_1, \ldots, x_{i-1}, X_i', x_{i+1}, \ldots, x_m)$\footnote{Denote $Z_1= (x^{(1)}, \ldots, x^{(i-1)}, {X^{(i)}}', x^{(i+1)}, \ldots, x^{(m)})\in \R^m$ where the superscript $(i)$ is used for coordinate index. The subscript is used to denote different samples.};
see for example \cite{reinert2005three}.
To make this a {continuous-time Markov process} generator, we wait {an} exponential(1)-{distributed time} before every change. }

{
This generator {induces}  a Stein operator for the target distribution as follows. {Here we take $d=1$ for clarity; the generalisation to other $d$ is straightforward.}
Let $f: {\R^m}\to \R$ and consider the expectation w.r.t. the one-step evolution of the Markov chain
\begin{align*} 
&\E_{-i} [f(x^{(1)}, \ldots, x^{(i-1)}, X^{(i)}, x^{(i+1)}, \ldots, x^{(m)})]
\\
&= \int f(x^{(1)}, \ldots, x^{(i-1)}, y, x^{(i+1)}, \ldots, x^{(m)}) \P(X^{(i)}=y| X^{(j)} = x^{(j)}, j \ne i).
\end{align*}
We now consider the index $i$ as the $i$-th coordinate of multivariate random variables in $\R^m$.
The conditional expectation here fixing all but the $i$-th \emph{coordinate} term
only depends on the 
uni-variate
conditional distribution 
$Q{^{(i)}}$ 
with $Q{^{(i)}} (A) = \P( X^{(i)} \in A | X^{(j)} = x^{(j)},  j \ne i) $. 
{Thus,}
the Stein operator induced from the Markov chain has the form
\begin{equation}\label{eq:glauber-stein-raw}
 \A f(z) =  \A^{(I)} f(z)    
\end{equation}
where
\begin{equation}\label{eq:glauber-stein-component}
    \A^{(i)} f({{\x}} ) = \E_{-i} [f(x^{(1)}, \ldots, x^{(i-1)}, X^{(i)}, x^{(i+1)}, \ldots, x^{(m)})
    ] - f({{\x}}).
\end{equation}
From the 
law of total expectation {it follows that} 
the {Stein}  identity holds.

\subsection{Step 2: marginal Stein operators}
{
In Eq.\,(\ref{eq:glauber-stein-component}), the expectation 
$$\E_{-i} [f(x^{{(i)}}, \ldots, x^{{(i-1)}}, X^{{(i)}}, x^{{(i+1)}}, \ldots, x^{{(m)}})] - f(x^{{(1)}}, \ldots, x^{{(m)}}) $$
can itself be treated via Stein's method, by finding a Stein operator $\mathcal{T}^{{(i)}}$ and a function $g$ such that $g=g_f$ solves the $\mathcal{T}^{(i)}$-Stein equation Eq.\,(\ref{eq:stein_equation}) for $f$; 
\begin{equation}\label{stein-t}
\mathcal{T}^{(i)} g (x) = 
\E_{-i} [f(x^{(1)}, \ldots, x^{(i-1)}, X^{(i)}, x^{(i+1)}, \ldots, x^{(m)})] - f(x^{(1)}, \ldots, x^{(m)}).
\end{equation} 
Fixing $x_j, j \ne i$ and setting $f^{(i)}(x) = 
f(x^{(1)}, \ldots, x^{(i-1)}, x, x^{(i+1)}, \ldots, x^{(m)})$, we view $\mathcal{T}^{(i)}$ as a Stein operator for a uni-variate distribution, acting on functions $g=g_{x^{(-i)}}:\mathbb{R} \rightarrow \mathbb{R}$. }
 
{ 
Summarising the approach, the Stein operator $\mathcal{A} $ acting on functions $f: \mathbb{R}^m \rightarrow\mathbb{R}$ 
underlying the non-parametric Stein operator is 
\begin{equation}\label{summ}
\mathcal{A} f(x^{(1)}, \ldots, x^{(m)}) = \mathcal{T}^{(I)} g_{f, x^{-I}}(x^{(I)})
\end{equation}
where $I \in [m]$ is a randomly chosen index.
{In view of Eq.\,(\ref{summ})  we take $g: \mathbb{R}^m \rightarrow\mathbb{R}$, write 
$g_{x^{(-i)}}(x): \mathbb{R} \rightarrow\mathbb{R}$ for the uni-variate function which acts only on the coordinate $i$ and fixes the other coordinates to equal $x^{(-i)}$, we {as} Stein operator (using the same letter $\mathcal{A}$ as before, which is  abuse of notation); 
\begin{equation*}
\mathcal{A} g(x^{(1)}, \ldots, x^{(m)}) = \mathcal{T}^{(I)} g_{x^{-I}}(x^{(I)}).
\end{equation*}
This formulation simplifies Eq.\,(\ref{summ}) in that we no longer have to consider the connection between $f$ an $g$.
}
} 
{The final step is to note that when we condition on the random index $I$, again a Stein operator is obtained, as follows. As 
\begin{equation}\label{eq:steinfinal}
\mathbbm{E}_I [\mathcal{A} g(x^{(1)}, \ldots, x^{(m)}) ]=  \frac1m \sum_{i=1}^m \mathcal{T}^{(i)} g_{x^{(-i)}}(x^{(i)}).
\end{equation}
As $\mathbb{E} [\mathcal{T}^{(i)} g_{X^{(-i)}}(X^{(i)})]=0$,
the Stein identity is satisfied. The operator in Eq.\,(\ref{eq:steinfinal}) is the Stein operator given in Eq.\,(\ref{summ2}).} 
The 
strategy {of} averaging over all coordinate terms $i\in[m]$ has also studied {in}  variational inference, 
via coordinate ascent variational inference (CAVI) \citep{bishop2006pattern} {which} 
focuses on latent variable inference. }

\section{Proofs and additional results}\label{app:proofs}

{Assuming that if $f \in \mathcal{H} $ then $-f \in \mathcal{H}$ we can assume that the supremum over the expectation is non-negative, and with Eq.\,\ref{eq:np-ksd},
\begin{eqnarray}
0 \le {\NPKSD}_t (P \| Q)  
&=& \sup_{f \in \B_1(\H)} \E_p[\widehat \A_{t,N}^B f] \nonumber 
\\
&=& \sup_f \{ \E_p \mathcal{A}_tf + \E_p [ \widehat \A_{t,N}^B  -  \mathcal{A}_t] f \}\nonumber  \\
&= & \sup_f \{ \E_p \mathcal{A}_tf + \E_p [ {\widehat \A}_{t,N}^B -  \widehat {\mathcal{A}}_{t,N} ] f   + \E_p[ f  (\widehat{s}^{(i)}_{t,N} - \log q_t') ] \} .
\label{eq:npdecomp}\end{eqnarray}}





{Here ${\widehat \A}_{t,N}$ is the Stein operator using the estimated conditional score function $\hat{s}_{t,N}$ with the estimation based on $N$ synthetic observations. We now assess the contribution to \ref{eq:npdecomp} which stems from estimating the score function.} {Note that here we only need to estimate a one-dimensional score function and hence the pitfalls of score estimation in high dimensions do not apply. We note however the contribution \cite{zhou2020nonparametric} for a general framework.}

Assume that we estimate the uni-variate conditional density ${q}_t^{(i)}$ 
based on $N$ samples. We assume that ${q}_t^{(i)}$  is differentiable, and we denote its score function by
$$ s_t^{(i)} (x) = \frac{({q}_t^{(i)})'(x)}{{q}_t^{(i)}(x)}.$$
{We next prove an extension of Proposition \ref{prop:Stein operators}.} 

\begin{proposition}\label{lem:Stein operatorsb}
Suppose that  for $ i\in [m]$,  $\widehat{s}_{N}^{(i)} $ is a consistent estimator of the {uni-variate} score function  $ s^{(i)}$. Let $\mathcal{T}^{(i)}$ be a Stein operator for the uni-variate differentiable probability distribution ${Q}^{(i)}$ of the generalised density operator form Eq.\,(\ref{opform}).
Let 
\begin{eqnarray*} 
{\widehat{\mathcal{T}}}^{(i)}_{N} g (x) &=& g'(x) + g(x)  \widehat{s}_{N}^{(i)}\\  \widehat{\mathcal{A}}g (\x)  &=& {\widehat{\mathcal{T}}}^{(I)}_{N} g_{ x^{(-I)}} (x^{(I)})  \quad \quad \mbox{ and }\\   \widehat{\mathcal{A}}_N g (\x)  &=& {\frac1m \sum_{i \in [m]} {\widehat{\mathcal{T}}}^{(i)}_{N} g_{ x^{(-i)}} (x^{(i)})}.
\end{eqnarray*} 
Then 
${\widehat{\mathcal{T}}}^{(i)}_{N} $ is a consistent estimator for ${{\mathcal{T}}}^{(i)}$, {and} 
$\widehat{\mathcal{A}}  $ {as well as $\widehat{\mathcal{A}}_N  $ are}
consistent estimator{s}  of $\mathcal{A}.$

\end{proposition} 

\begin{proof}
Take a fixed $x$.
As $\widehat{s}_{ N}^{(i)}$ is a consistent estimator of $s^{(i)}$, it holds that for any $\epsilon>0$ {and for any $x$ in the range of $s^{(i)}$},
$$\mathbb{P}( | \widehat{s}_{ N}^{(i)} {(x)}
 - {s}^{(i)}(x) | > \epsilon) \rightarrow 0$$
as $N\rightarrow \infty$. Here $\omega$  denotes the random element for the estimation, which is implicit in $\widehat{q}_{ N}$. 
On the set 
$$A_\epsilon = \left\{| \widehat{s}_{ N}^{(i)} {(x)} - {s}^{(i)}(x) |\le \epsilon 
\right\}$$ 
we have that 
$$
| {\widehat{\mathcal{T}}}^{
{(i)}}_{N} g (x)
- {{\mathcal{T}}}^{(i)} g (x)| \le \epsilon f(x) . 
$$
For every fixed $x$ this expression tends to 0 as $\epsilon \rightarrow 0$.
Hence consistency of $\widehat{T}_{N}$ follows. The last {two} assertions follow  immediately from Eq.\,(\ref{summ}) {and Eq.\,(\ref{eq:steinfinal})}. 
\end{proof}

\subsection{Asymptotic behaviour of NP-KSD}\label{app:asymp}

{Here we assess the asymptotic behaviour of NP-KSD$^2$. 
{{W}ith $\boldsymbol s_t$ denoting the conditional score function, 
$$ \NPKSD^2_t(G \| p) 
= \E_{\x, \x' \sim p} 
\left\langle \mathcal{A}_{Q_t} k(\x, \cdot),  \mathcal{A}_{Q_t} k(\x', \cdot)\right\rangle_\mathcal{H} 
$$
{where $ \mathcal{A}_{Q_t}k(\x, \cdot) = \mathcal{A}_{t}k(\x, \cdot)$ can be written as } 
{\begin{eqnarray}
\label{eq:decomp_mv}
    \mathcal{A}_{t} k(\x, \cdot)
&= &\frac{1}{m} \sum_{{i \in [m]}} \left\{  \mathcal{A}_{\widehat{Q}^{(i)}_t} k(\x, \cdot) + k(\x, \cdot) (\widehat{s}^{(i)}_{t,{N}} - s^{(i)}_t) \right\} \nonumber \\ &=& { \frac{1}{m} \sum_{i \in [m]}}\left\{  \frac{\partial }{\partial x^{(i)}}k(\x, \cdot) + k(\x, \cdot) {\widehat s}^{(i)}_{t,{N}} + k(\x, \cdot) (\widehat{s}^{(i)}_{t,{N}} - s^{(i)}_t) \right\}. 
\end{eqnarray}}
}
Recall that KSD$^2$ is given in Eq.\,(\ref{eq:KSDequiv}) by  
$$
\mathrm{KSD}^2(q\|p) = {\E}_{\x,\tilde{\x} \sim p} [\langle \T_q k(\x,\cdot), \T_q k(\tilde \x,\cdot)\rangle_{\H}],
$$
where $\KSD(q \| p)$  is a deterministic quantity which under weak assumption vanishes when $p=q$.}
{Moreover, 
$$
{\KSD}_t^2(q_t\|p) 
= {\E}_{\x,\tilde{\x} \sim p} [\langle \A_t k(\x,\cdot), \widehat \A_t k(\tilde \x,\cdot)\rangle_{\H} . 
$$
Disentangling this expression in general is carried out using Eq.\,(\ref{summ2}).

\begin{remark} For Gaussian kernels $k=k_G$ used in this paper, we can exploit its factorisation; 
\begin{eqnarray*}
k_G(\x, \tilde{\x} ) &=& \exp \left\{ - \frac{1}{2 \sigma^2} \sum_{i=1}^m \left(x^{(i)} - {\tilde{x}}^{(i)} \right)^2  \right\} 
= \prod_{i=1}^m \exp \left\{ - \frac{1}{2 \sigma^2}  \left(x^{(i)} - {\tilde{x}}^{(i)} \right)^2  \right\}.
\end{eqnarray*}
In this situation, taking $g_{{\x}} (\cdot) = k_G(\x, \cdot )$,
with $\cdot$ denoting an element in $\R^m$, 
gives 
$$g_{{{\x}}; x^{(-i)}} ({\cdot})  
= \exp \left\{ - \frac{1}{2 \sigma^2} \sum_{j: j \ne i}^m \left(x^{(j)} - (\cdot)^{(j)} \right)^2  \right\} \exp \left\{ - \frac{1}{2 \sigma^2}  \left(x^{(i)} -  (\cdot)^{(i)}\right)^2  \right\}. $$
For the operator ${\mathcal{T}}_q^{(i)}$ in Eq.\,(\ref{opform}) we have 
\begin{eqnarray*}
{\mathcal{T}}_q^{(i)} g_{x^{(-i)}} ({\cdot})  
&=& {\exp \left\{ - \frac{1}{2 \sigma^2} \sum_{j=1}^m \left( x^{(j)} - (\cdot)^{(j)} \right)^2  \right\} }    \left( \frac{1}{\sigma^2} (x^{(i)} -  (\cdot)^{(i)}) 
+ 
(\log q_{t(x^{(-i)})})' (x^{(i)}) \right). 
\end{eqnarray*} 
Thus, the operator $\A_t$ decomposes as 
{
\begin{eqnarray*} \frac1m \sum_{i=1}^m
{\mathcal{T}}_q^{(i)} g_{x^{(-i)}} ({\cdot}) 
&=&  \exp \left\{ - \frac{1}{2 \sigma^2} \sum_{j=1}^m \left(x^{(j)} - (\cdot)^{(j)} \right)^2  \right\} \\
&& \frac1m \sum_{i=1}^m \left\{  \frac{1}{\sigma^2}  \left(x^{(i)} -  (\cdot)^{(i)}\right) 
+ 
(\log q_{t(x^{(-i)})})' (x^{(i)}) \right\}. 
\end{eqnarray*} 
}

\end{remark} 
}

{For consistency, in our setting the Stein operators are only applied to the observations $\z_1, \ldots, \z_n$ and hence for our applications pointwise consistent estimation suffices, in the sense that 
that for $i=1, \ldots, m$, {$\widehat{s}_{t,N}^{(i)} = \widehat{s}_{t,N}^{(i)} (x^{(i)})$} is a consistent estimator of the {uni-variate} score function  $ s_t^{(i)} = \{\log (q(x^{(i)})| t (x^{{(-i)}})\}'$. Score matching estimators often satisfy not only consistency but also asymptotic normality, see for example  \cite{song2020sliced}. Such an assumption is required for Theorem \ref{th:KSDconv}; recall that we use the notation 
$ \hat{\boldsymbol s}_{t,N}= (\hat{s}_{t,N} (x^{(i)}), i \in [m]) $.} 
{To prove Theorem \ref{th:KSDconv} we re-state it  for convenience.}

{
\begin{theorem}\label{th:KSDconvb} 
Assume that the score function estimator vector $ \hat{s}_{t,N}$ is asymptotically normal with mean $0$ and covariance  matrix $N^{-1} \Sigma_{s}$. Then 
${\NPKSD}_t^2(G\|p)$ converges in probability to ${\KSD}_t^2(q_t\|p)$ at rate at least $\min(B^{-\frac12}, N^{-\frac12})$. 
\end{theorem}
}
{
{\bf{Proof.}}
We have from Eq.\,(\ref{eq:KSDequiv})
$$
{\NPKSD}_t^2(G\|p) = {\E}_{\x,\tilde{\x} \sim p} [\langle \widehat \A^B_{t,N} k(\x,\cdot), \widehat \A^B_{t,N} k(\tilde \x,\cdot)\rangle_{\H}. 
$$
Expanding this expression, with $\A_t$ denoting the score Stein operator in Eq.\,\ref{summ2} for the conditional distribution $q_t$, 
\begin{eqnarray*}
{\NPKSD}_t^2(G\|p)&=&  {\E}_{\x,\tilde{\x} \sim p} [\langle \A_t k(\x,\cdot),  \A_t k(\tilde \x,\cdot)\rangle_{\H} \\
&&+  {\E}_{\x,\tilde{\x} \sim p} [\langle ( \widehat \A^B_{t,N}  -\A_t) k(\x,\cdot),  \A_{t} k(\tilde \x,\cdot)\rangle_{\H}
\\
&&+  {\E}_{\x,\tilde{\x} \sim p} [\langle\A_{t} k(\x,\cdot), ( \widehat \A^B_{t,N} -\A_t) k(\tilde \x,\cdot)\rangle_{\H}
\\
&&+  {\E}_{\x,\tilde{\x} \sim p} [\langle  (\widehat \A^B_{t,N} - \A_t) k(\x,\cdot), \widehat (\A^B_{t,N} -\A_t) k(\tilde \x,\cdot)\rangle_{\H}\\
&=& \KSD^2(q_t\| p) + 2 \,  {\E}_{\x,\tilde{\x} \sim p} [\langle ( \widehat \A^B_{t,N}  -\A_t) k(\x,\cdot), \A_t k(\tilde \x,\cdot)\rangle_{\H}\\
&&+  {\E}_{\x,\tilde{\x} \sim p} [\langle  (\widehat \A^B_{t,N} - \A_t) k(\x,\cdot), (\widehat \A^B_{t,N} -\A_t) k(\tilde \x,\cdot)\rangle_{\H}
\end{eqnarray*}
where we used the symmetry of the inner product in the last step. Now, for any function $g$ for which the expression is defined, 
\begin{eqnarray}\label{eq:sumofoper}
(\widehat \A^B_{t,N} - \A_t) g(\x)
&=& (\widehat \A^B_{t,N} - \widehat \A_{t,N}) g(\x) + 
(\widehat \A_{t,N} - \A_t) g(\x)
\end{eqnarray}
{recalling that} ${\widehat \A}_{t,N}$ is the Stein operator using the estimated conditional score function $\hat{s}_{t,N}$ with the estimation based on $N$ synthetic observations.}

{To analyse Eq.\,\ref{eq:sumofoper} we first consider $(\widehat \A_{t,N} - \A_t) g(\x)$;
\begin{eqnarray}\label{eq:expansion} 
{\widehat{\mathcal{A}}_{t,N}} g(\x)
- \mathcal{A}_{t} g(\x) = \frac1m \sum_{i=1}^m g(x^{(i)}) (\widehat{s}_{t,N}^{(i)}   (x^{(i)})- s_t^{(i)} (x^{(i)})) . 
\end{eqnarray}
{We note that} it suffices to 
assume that for $i=1, \ldots, m$, {$\widehat{s}_{t,N}^{(i)} = \widehat{s}_{t,N}^{(i)} (x^{(i)})$} is a consistent estimator of the {uni-variate} score function  $ s_t^{(i)} = \{\log (q(x^{(i)})| t (x^{(-i)})\}'$. Score matching estimators often satisfy not only consistency but also asymptotic normality, see for example  \cite{song2020sliced}.}  
{If for $x^{(1)}, \ldots, x^{(m)}$ the vector $ \hat{\boldsymbol s}_{t,N}= (\hat{s}_{t,N} (x^{(i)}), i \in [m]) $ is asymptotically normal with mean $0$ and covariance  matrix $N^{-1} \Sigma_{s}$ then it follows from Eq.\,\ref{eq:expansion} that, asymptotically, 
$\sqrt{N}  (\widehat{\mathcal{A}}_{t,N} g(\x)
- \mathcal{A}_{q_t} g(\x))$ has a multivariate normal distribution and,  in particular, $(\widehat{\mathcal{A}}_{t,N} g(\x)
- \mathcal{A}_{q_t} g(\x))$ has fluctuations of the order $N^{-\frac12}$. 
} 

{For the  term $(\widehat \A^B_{t,N} - \widehat \A_{t,N}) g(\x)$ in Eq.\,\ref{eq:sumofoper}
we have 
\begin{eqnarray*}
(\widehat \A^B_{t,N} - \widehat \A_{t,N}) g(\x)
&=& \frac1B \sum_{b=1}^B {{\mathcal{T}}_{t,N}}^{(i_b)} g(x) - {\T_{t,N}} g(\x)\\
&=& \sum_{i=1}^m 
\left\{\frac1B \sum_{b=1}^B {{\mathcal{T}}_{t,N}}^{(i_b)} g(\x) \mathds{1}(i_b = i) - \frac1m {{\mathcal{T}}_{t,N}}^{(i)} g(\x) \right\} .
\end{eqnarray*}
Let $k_i = \sum_{b=1}^B {\mathds{1}}(i_b = i)$ the number of times that $i$ is re-sampled. Then $ \mathbb{E}(k_i) = \frac{B}{m}$ and we have 
\begin{eqnarray*}
(\widehat \A^B_{t,N} - \widehat \A_{t,N}) g(\x)
&=& \sum_{i=1}^m {\mathcal{T}}_{t,N}^{(i)} g(\x) 
\left\{\frac1B k_i - \frac1m \right\} \\
&=& \frac1B \sum_{i=1}^m {\mathcal{T}}_{t,N}^{(i)} g(\x) 
\left\{ k_i -\mathbb{E}(k_i)\right\} 
.
\end{eqnarray*}
This term is known to be approximately mean zero normal with finite variance  $\Sigma(\hat{\boldsymbol s}_{t,N}; g)$ (which depends on $\hat{\boldsymbol s}_{t,N}$ and $g$)  of order $B^{-1}$, 
see for example \cite{holmes2004stein}, where an explicit bound on the distance to normal is provided. This asymptotic normality holds for the operator given the estimated conditional score function. As the bootstrap samples are drawn independently of the score function estimator, without conditioning, the unconditional distribution is a mixture of normal distributions. For an estimator $\hat{\boldsymbol s}_{t,N}$ which is asymptotically normally distributed, the variances  $\Sigma(\hat{\boldsymbol s}_{t,N}; g)$ will converge to  $\Sigma(\boldsymbol{s}_t; g)$.} 

{Thus, with Eq.\,\ref{eq:sumofoper}, 
\begin{eqnarray*}
\lefteqn{{\E}_{\x,\tilde{\x} \sim p} [\langle ( \widehat \A^B_{t,N}  -\A_t) k(\x,\cdot), \A_t k(\tilde \x,\cdot)\rangle_{\H}] }\\ 
&=& {\E}_{\x,\tilde{\x} \sim p} 
[\langle  (\widehat \A^B_{t,N} - \widehat \A_{t,N}) k(\x,\cdot), \A_t k(\tilde \x,\cdot)\rangle_{\H} ] 
+ {\E}_{\x,\tilde{\x} \sim p} 
[\langle (\widehat \A_{t,N} - \A_t)k(\x,\cdot), \A_t k(\tilde \x,\cdot)\rangle_{\H} ]
\end{eqnarray*}
with the first term is approximately a variance mixture of mean zero normals tending to $0$ in probability at rate at least $B^{-\frac12} $ as $B \rightarrow \infty$, and the second term approximately a mean zero normal variable tending to 0 in probability at rate at least $N^{-\frac12}$ as $N \rightarrow \infty$.}

{It remains to consider 
$$  {\E}_{\x,\tilde{\x} \sim p} [\langle  (\widehat \A^B_{t,N} - \A_t) k(\x,\cdot), (\widehat \A^B_{t,N} -\A_t) k(\tilde \x,\cdot)\rangle_{\H}.$$ 
With Eq.\,\ref{eq:sumofoper} we have
\begin{eqnarray}
\lefteqn{ {\E}_{\x,\tilde{\x} \sim p} [\langle  (\widehat \A^B_{t,N} - \A_t) k(\x,\cdot), (\widehat \A^B_{t,N} -\A_t) k(\tilde \x,\cdot)\rangle_{\H}} \nonumber \\
 &=&   {\E}_{\x,\tilde{\x} \sim p} [\langle  (\widehat \A^B_{t,N} - \widehat \A_{t,N}) k(\x,\cdot), (\widehat \A^B_{t,N} -\widehat \A_{t,N}) k(\tilde \x,\cdot)\rangle_{\H} \label{term1} \\
 &&+  {\E}_{\x,\tilde{\x} \sim p} [\langle  (\widehat \A^B_{t,N} - \widehat \A_{t,N}) k(\x,\cdot), (\widehat \A_{t,N} -\A_t) k(\tilde \x,\cdot)\rangle_{\H} \label{term2}\\
 &&+  {\E}_{\x,\tilde{\x} \sim p} [\langle  (\widehat \A_{t,N} - \A_t) k(\x,\cdot), (\widehat \A^B_{t,N} -\widehat \A_{t,N}) k(\tilde \x,\cdot)\rangle_{\H} \label{term3}
 \\
 &&+  {\E}_{\x,\tilde{\x} \sim p} [\langle  (\widehat \A_{t,N} - \A_t) k(\x,\cdot), (\widehat \A_{t,N} -\A_t) k(\tilde \x,\cdot)\rangle_{\H}. \label{term4}
\end{eqnarray}
}

In \cite{xu2021stein}, Proposition 2, the following result is shown, using the notation as above. 

Let    
$$Y  = \frac{1}{B^2}  \sum_{s, t \in {[m]} } ( k_s k_t - \E (k_s k_t) ) h_x (s, t).$$
Assume that  $h_x$ is  bounded and that $Var(Y)$ is non-zero. 
Then if $Z$ is mean zero normal with variance $Var(Y)$, there is an {explicitly computable}  constant $C>0$ such that for all three times  continuously differentiable functions $g$ with bounded derivatives up to order 3, 
$$
| \E [ g(Y) ] - \E [g(Z)] \le \frac{C}{B}. 
$$ 
{Moreover, using Equations (17)-(21) from \cite{ouimet2021general}, it is easy to see that $Var(Y)$ is of the order $B^{-1}$. Hence, Term (\ref{term1}) tends to  0 in probability at rate at least $B^{-1}$. Similarly, using that the bootstrap sampling is independent of the score function estimation, Terms (\ref{term2}) and  (\ref{term3}) tend to  0 in probability at rate at least $(NB)^{-\frac12}$. For Term (\ref{term4}), from Eq.\,(\ref{eq:expansion}),
\begin{eqnarray*} 
\lefteqn{ {\E}_{\x,\tilde{\x} \sim p} [\langle  (\widehat \A_{t,N} - \A_t) k(\x,\cdot), (\widehat \A_{t,N} -\A_t) k(\tilde \x,\cdot)\rangle_{\H}]}\\
&=& \frac{1}{m^2} \sum_{i=1}^m \sum_{j=1}^m 
{\E}_{\x,\tilde{\x} \sim p} [\langle (\widehat{s}^{{(i)}}_{t,N}  (x^{{(i)}})- s^{{(i)}}_t (x^{{(i)}})) k(x^{{(i)}},\cdot), (\widehat{s}^{{(j)}}_{t,N}  (x^{{(j)}})- s^{{(j)}}_t (x^{{(j)}})) k(x_j,\cdot)\rangle_{\H}].
\end{eqnarray*}
If $\hat{\boldsymbol s}_{t,N}$ is approximately normal as hypothesised, then the inner product is approximately a covariance of order $N^{-1}$, and hence the overall contribution from 
Term (\ref{term4}) is of order at most $N^{-1}$. This finishes the proof. \hfill $\Box$
}


\section{Additional experimental details and results}\label{app:exp}

\subsection{Additional experiments}

\begin{table}[t!]
\centering
    \begin{tabular}{c|cccc|c}
    \toprule
      {NP-KSD} & B=5 & B=10 & B=20 & B=40 & (MMDAgg) \\ 
      \midrule
      Runtime(s) & 4.65 & 6.56 & 8.43 & 10.44 & 5.02 \\ \hline 
    Rejection Rate & 0.24 & 0.40 & 0.51 & 0.55 & 0.27 \\
    \bottomrule
    \end{tabular}
    \vspace{0.3cm}
  \caption{Computational runtime for various re-sample size B: observed sample size $n=100$; bootstrap size  $b=200$; dimension $m=40$. {The r}ejection rate {is used for} 
  power comparison; {higher rejection rates indicate higher power}.}    \label{tab:runtime}
  \end{table}
\paragraph{Runtime}
The computational runtime for each tests are shown in \cref{tab:runtime}. \textbf{MMDAgg} runtime is also shown as a comparison. 
From the result, we can see that
NP-KSD runs generally slower than permutation-based test, i.e. \textbf{MMDAgg}. This is mainly due to 
the learning of conditional score functions and the Monte-Carlo based  bootstrap procedure. As the re-sample size $B$ increase, NP-KSD test requires longer runtime. However, the rejection rate  $B=20$ is approaching to that of $B=40$ ( similiar observations also shown in \ref{fig:MoG_resample}). \textbf{MMDAgg} generally has faster computation due to permutation procedure of the test. However, it has lower test power, which is only comparable to that of $B=5$ at which the runtime advantage is not that obvious.

\paragraph{Training 
on synthetic distributions}
We also train the deep generative models on the synthetic distributions studied in {Section \ref{sec:exp}} 
and perform model assessment on the trained models. We consider the 
standard Gaussian and Mixture of two-component Gaussian problems. We train a generative adversarial network with multi-layer perceptron (\textbf{GAN\_MLP})\footnote{DCGAN studied in the main text is particularly useful for the (high-dimensional) image dataset due to the convolutional neural network (CNN) layers{; DCGAN} 
is not applicable for the problem in $\R^2$.} and a variational auto-encoder (VAE) \citep{kingma2013auto}.
Noise-constrastive score network \textbf{NCSN} is also trained to learn the score function followed by annealed Langevin dynamics \citep{song2019generative, song2020improved}. The trainings are done via Adam optimiser \citep{kingma2014adam} with adaptive learning rate. The rejection rates are reported in \cref{tab:train_synthetic}.

As shown in \cref{tab:train_synthetic}, uni-modal Gaussian distribution is easier to be learned by the generative modelling procedures, as compared to the two-component Mixture of Gaussian (MoG). As a result, the $\NPKSD$\_m testing procedure shows higher rejection rate on trained MoG generative models compared to that of Gaussian.\footnote{We note that $\NPKSD$ and $\NPKSD$\_m with  summary statistics {taken to be the mean are}
equivalent in the two-dimensional problem.}.
However, as these deep models are not designed for training and sampling the simple low-dimensional distribution, it is not surprising the procedure produce samples that not pass the $\NPKSD$ tests.

Inspired from the settings in \citet{gorham2017measuring}, where KSD {is} used to measure sample quality, we also apply $\NPKSD$ tests on the Stochastic Gradient Langevin Dynamics (SGLD) \citep{welling2011bayesian} sampling procedure studied in \citet{gorham2017measuring}; {in \citet{gorham2017measuring}, SGLD is referred to as 
Stochastic Gradient Fisher Scoring (SGFS).} SGLD is capable of sampling uni-modal distributions, while it can have problems sampling multi-modal data. {T}he rejection rates{ shown} in \cref{tab:train_synthetic} 
{are} slightly higher 
than {the} test level for MoG, while {the} type-I error is well controlled for the Gaussian case. Generated samples from SGLD are visualised \cref{fig:sgld_samples}, {illustrating that the SGLD samples look plausible for the Gaussian model, but less so for the MoG model.} 

\begin{table}[h!]
\centering
\vspace{0.35cm}
    \begin{tabular}{c|ccc|cc}
    \toprule
  {} & GAN\_MLP & VAE & NCSN & SGLD & Real \\ 
      \midrule
  Gaussian & 0.36 & 0.61 & 0.25 & 0.06 & 0.03\\
  \midrule
  MoG & 0.78 & 0.92 & 0.45 & 0.12 & 0.04 \\
    \bottomrule
    \end{tabular}
  \vspace{0.35cm}
\caption{$\NPKSD$\_m rejection rate: observed sample size $n=100$; bootstrap size is $200$. {Here a low rejection rate indicates a good type-1 error. NCSN performs best, among deep generative models, on both tasks but still has a very high rejection rate. {SGLD outperforms the deep generative models.}}}    \label{tab:train_synthetic}
  \end{table}

\begin{figure}[htp!]
    \centering
    \includegraphics[width=0.66\textwidth]{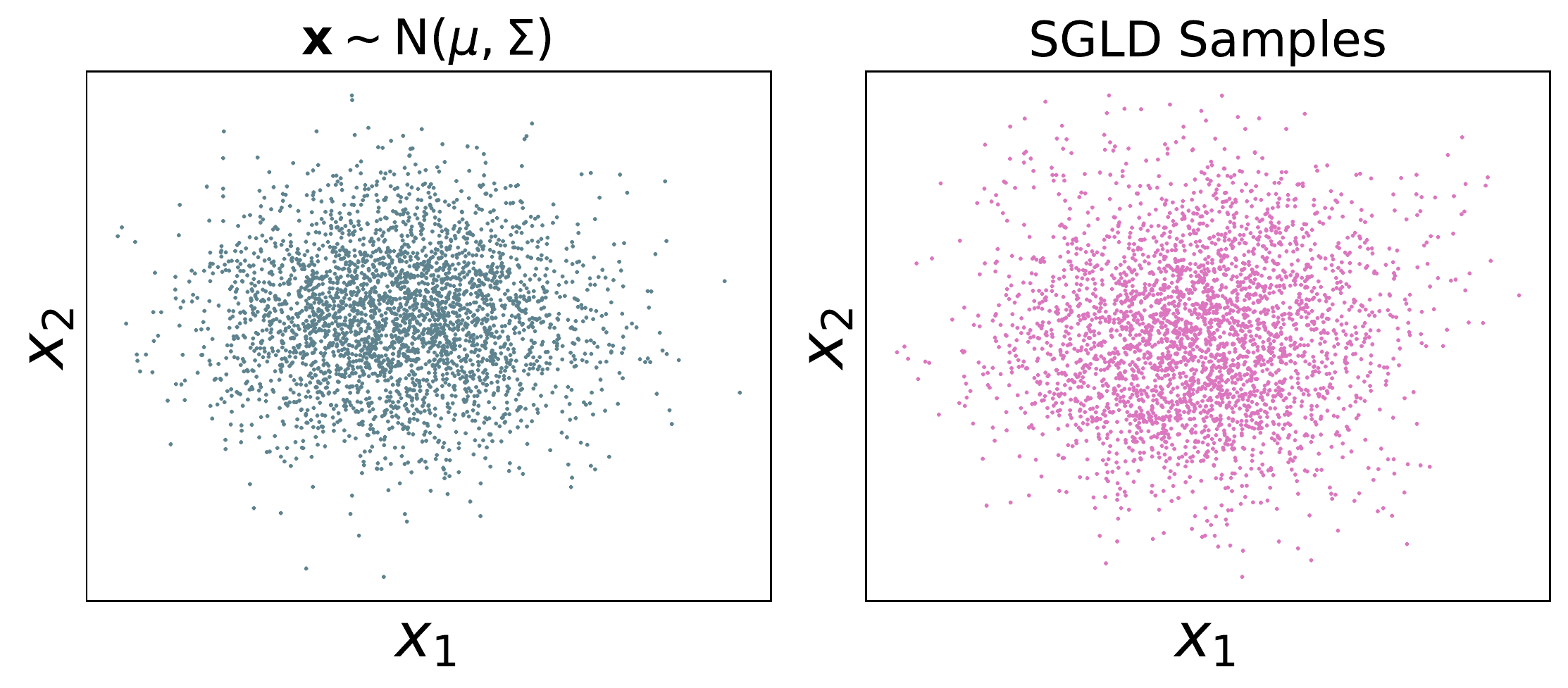} 
    
    \vspace{0.5cm}
    \includegraphics[width=0.66\textwidth]{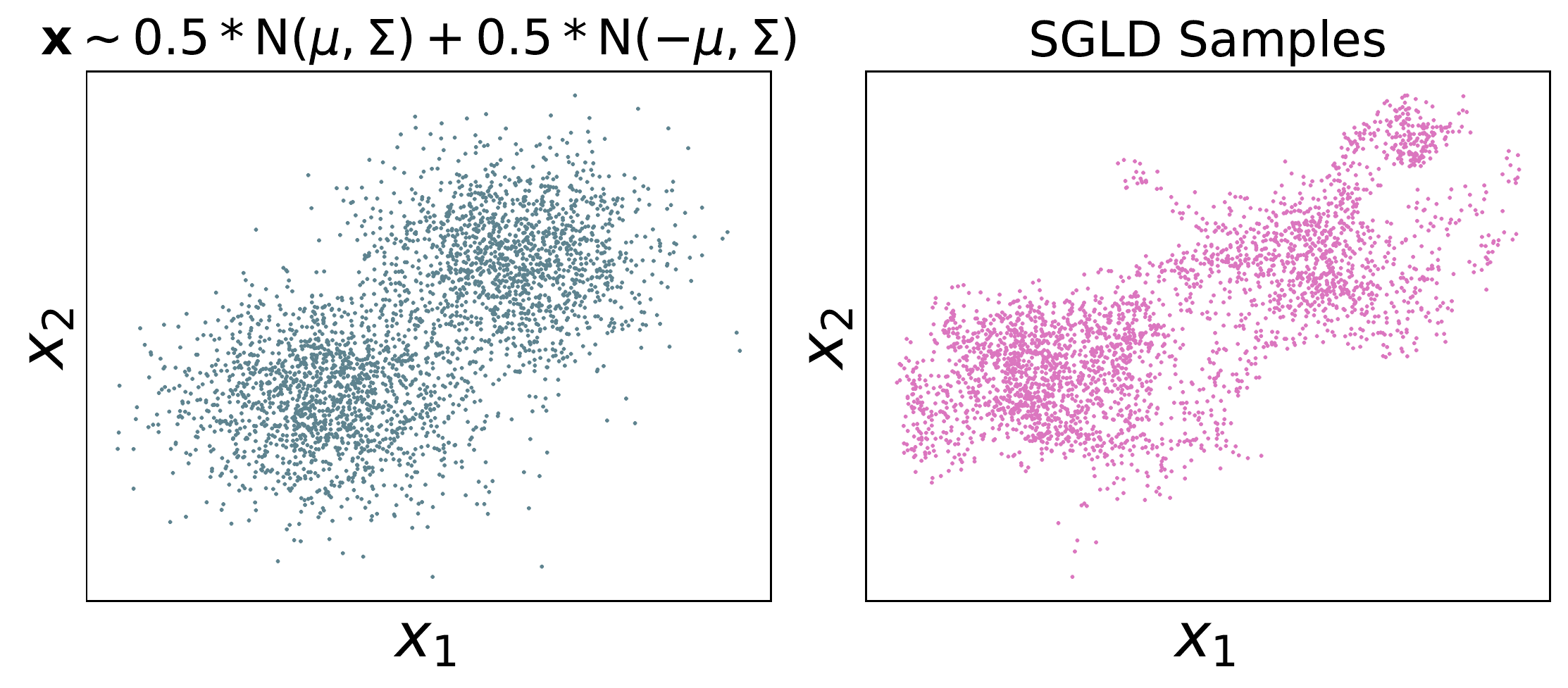}
    \caption{Visualisation of  samples generated from Stochastic Gradient Langevin Dynamics (SGLD); {top: Gaussian model, bottom: MoG.}}
    \label{fig:sgld_samples}
\end{figure}

\newpage
\subsection{Data visualisation}\label{app:visualisation}
We show samples from the MNIST and CIFAR10 dataset, together with 
samples from 
trained generative models, in \cref{fig:exp_mnist} and \cref{fig:cifar10},  respectively.

\begin{figure}[htp!]
\centering
    \subfigure[Real MNIST samples
    ]{
    \includegraphics[width=0.4\textwidth]{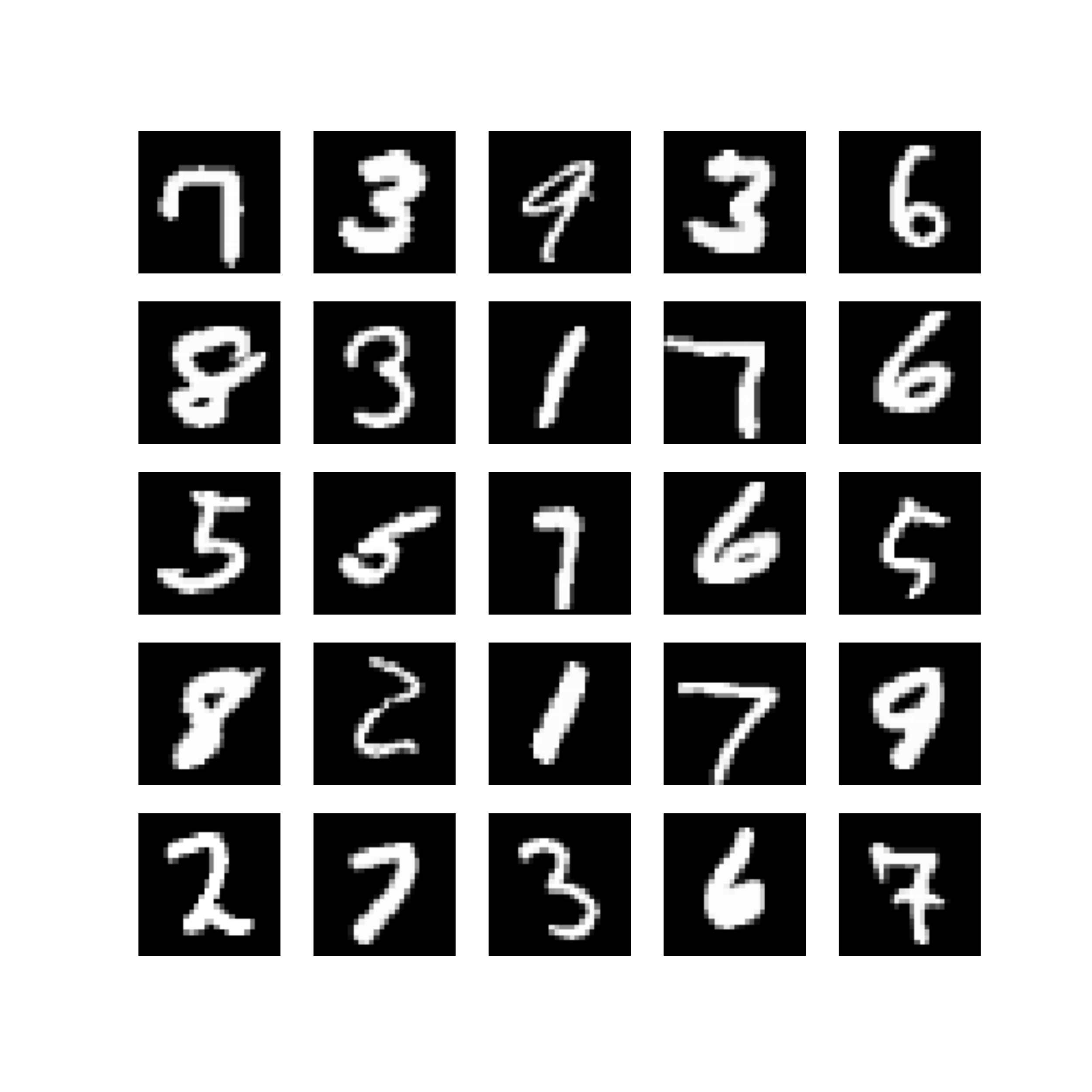}\label{fig:mnist_real}}

        \subfigure[NCSN samples
    ]{
    \includegraphics[width=0.4\textwidth]{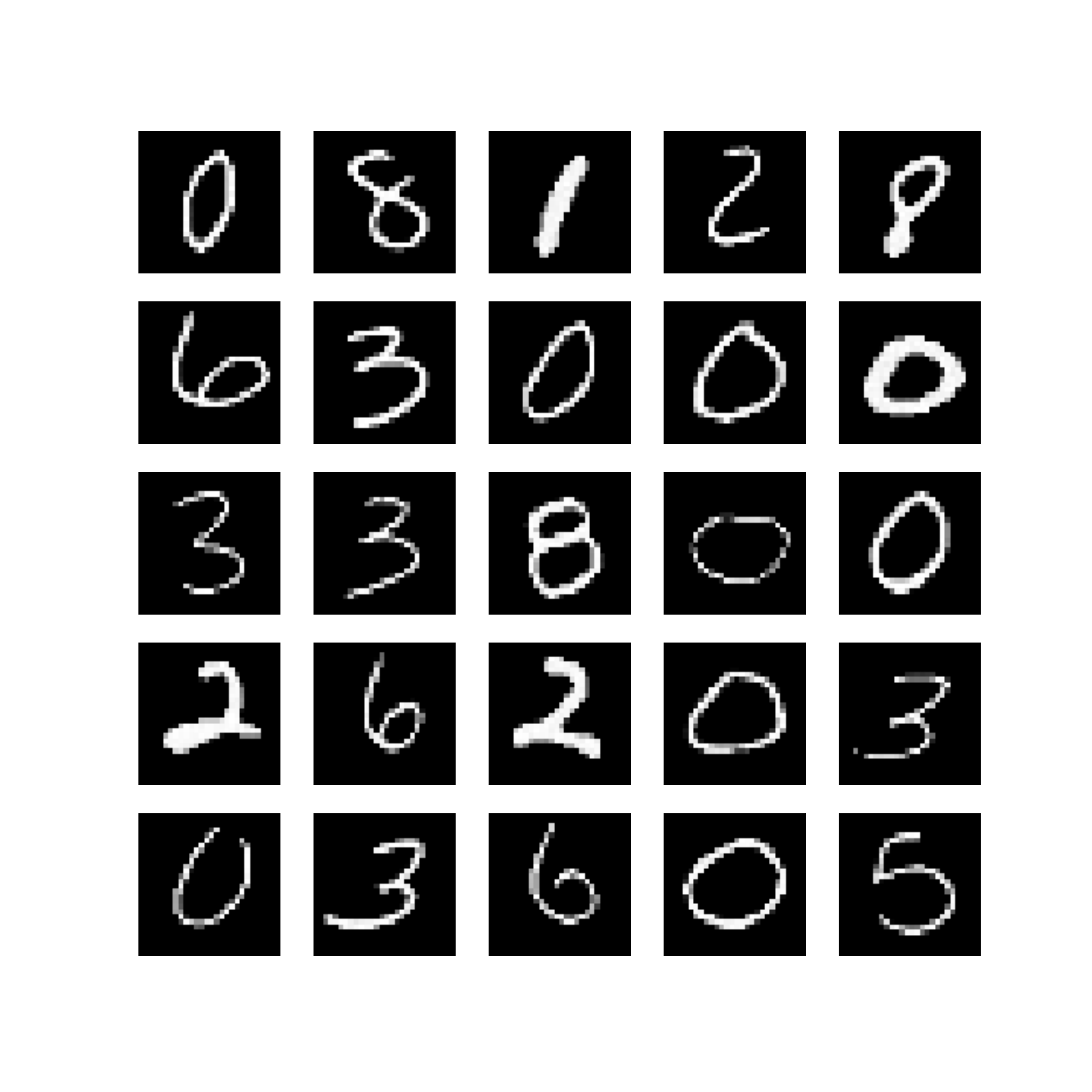}\label{fig:mnist_langevin}}
        \subfigure[DCGAN samples
    ]{
    \includegraphics[width=0.4\textwidth]{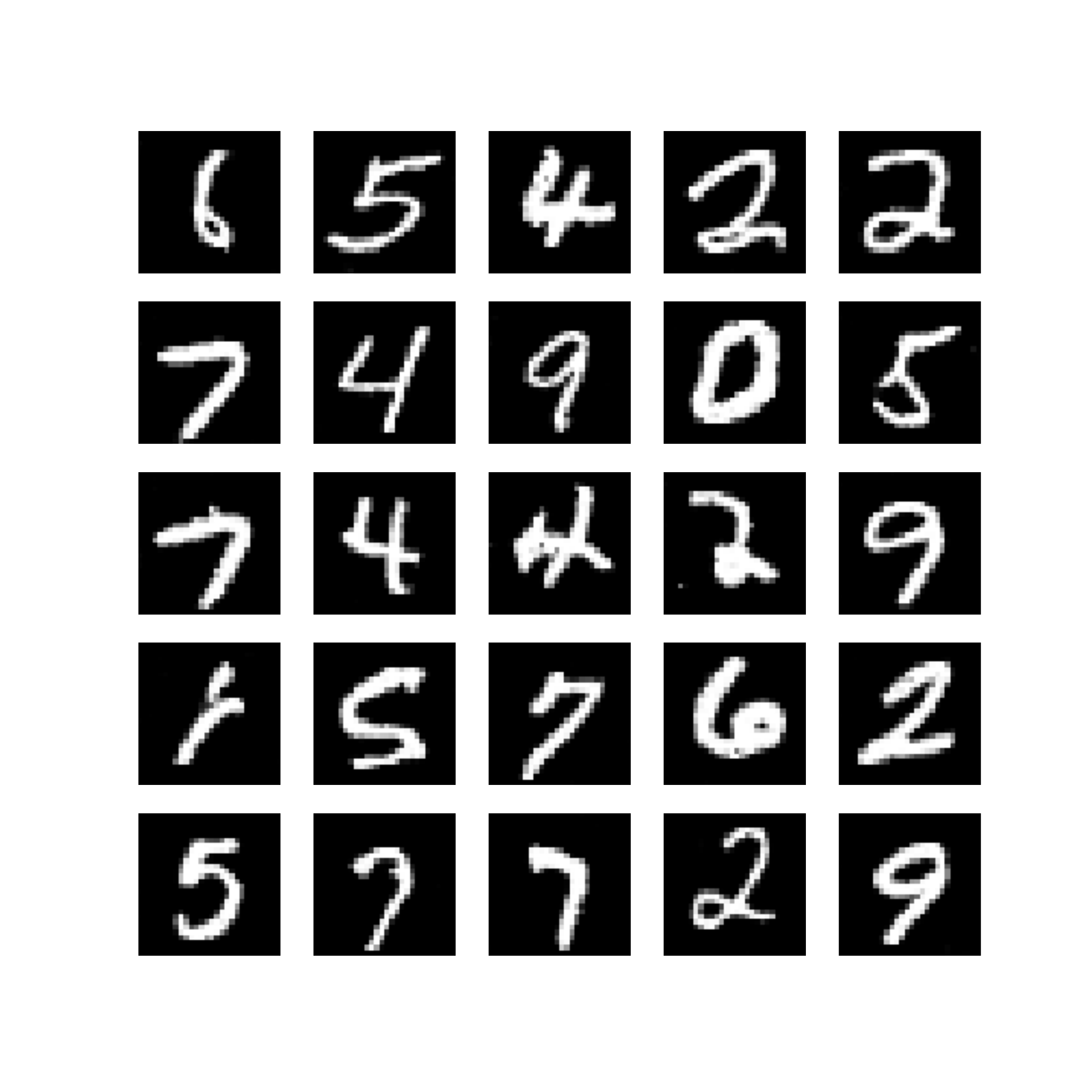}}\label{fig:mnist_dcgan}    
    
    \subfigure[GAN samples]{
    \includegraphics[width=0.4\textwidth]{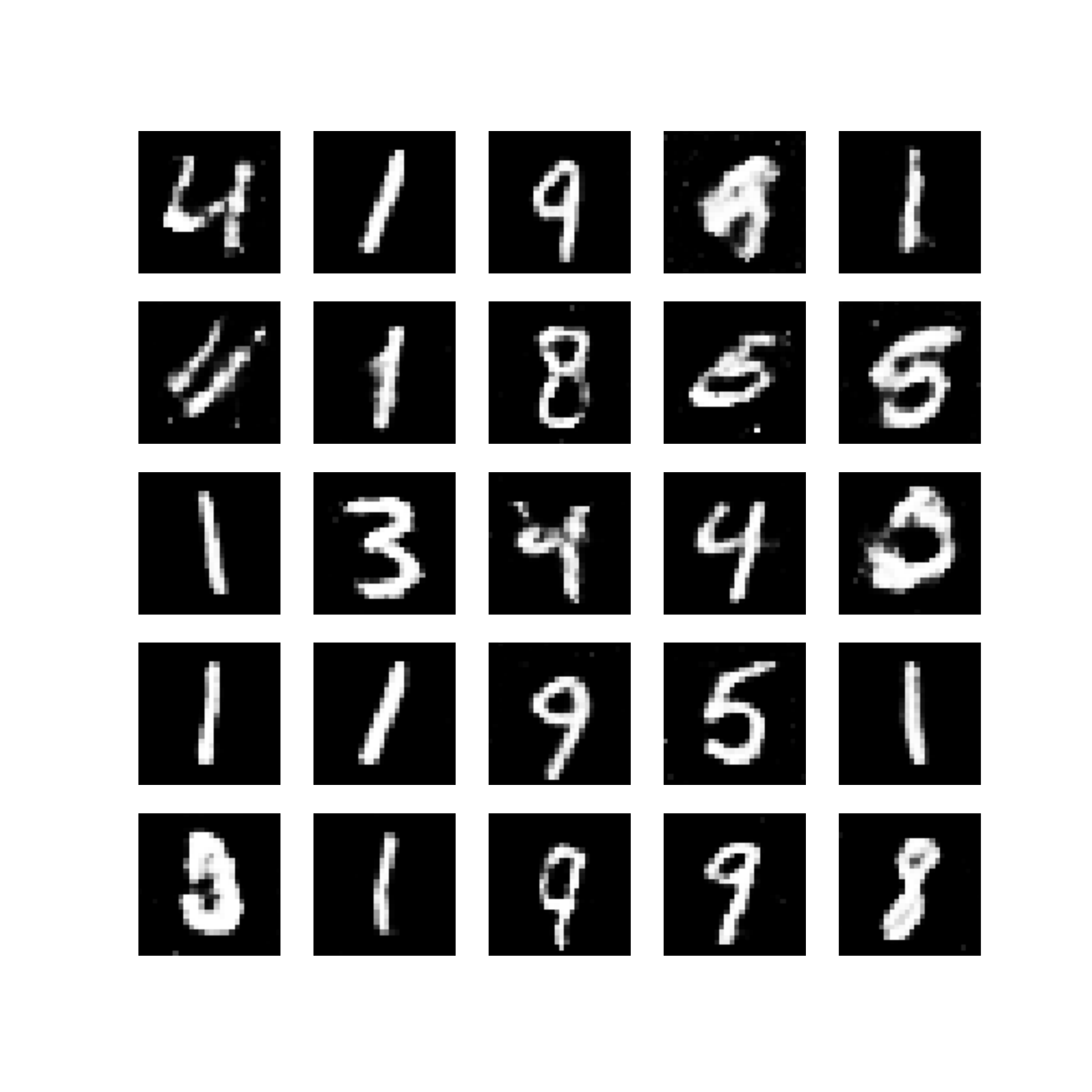}}\label{fig:mnist_gan}    \subfigure[VAE samples
    ]{
    \includegraphics[width=0.4\textwidth]{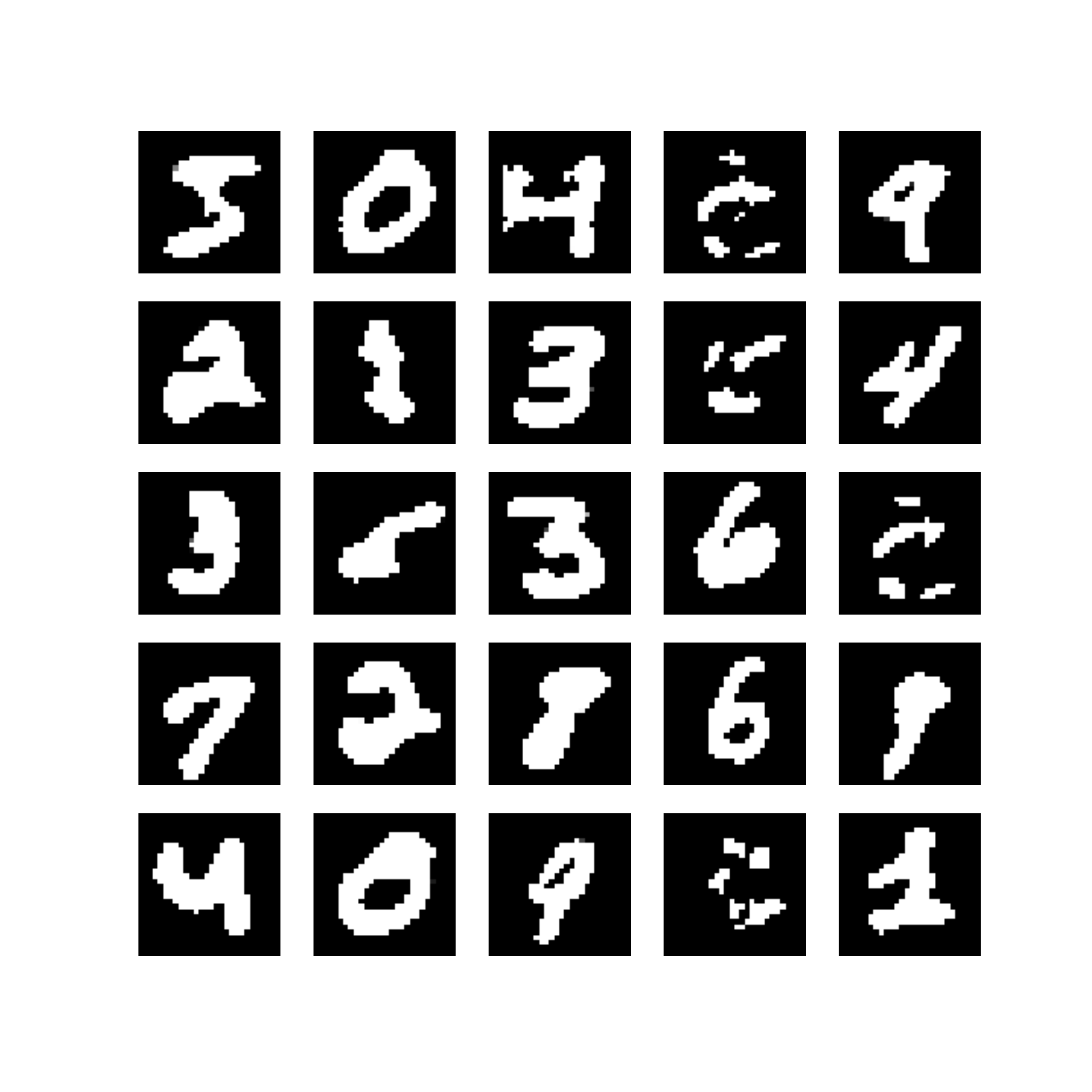}}\label{fig:mnist_vae}
    
    \caption{MNIST samples 
}
    \label{fig:exp_mnist}
\end{figure}

\begin{figure}[t!]
\centering
\hspace{-.2cm}
    \subfigure[Real samples]{
    \includegraphics[width=0.52\textwidth]{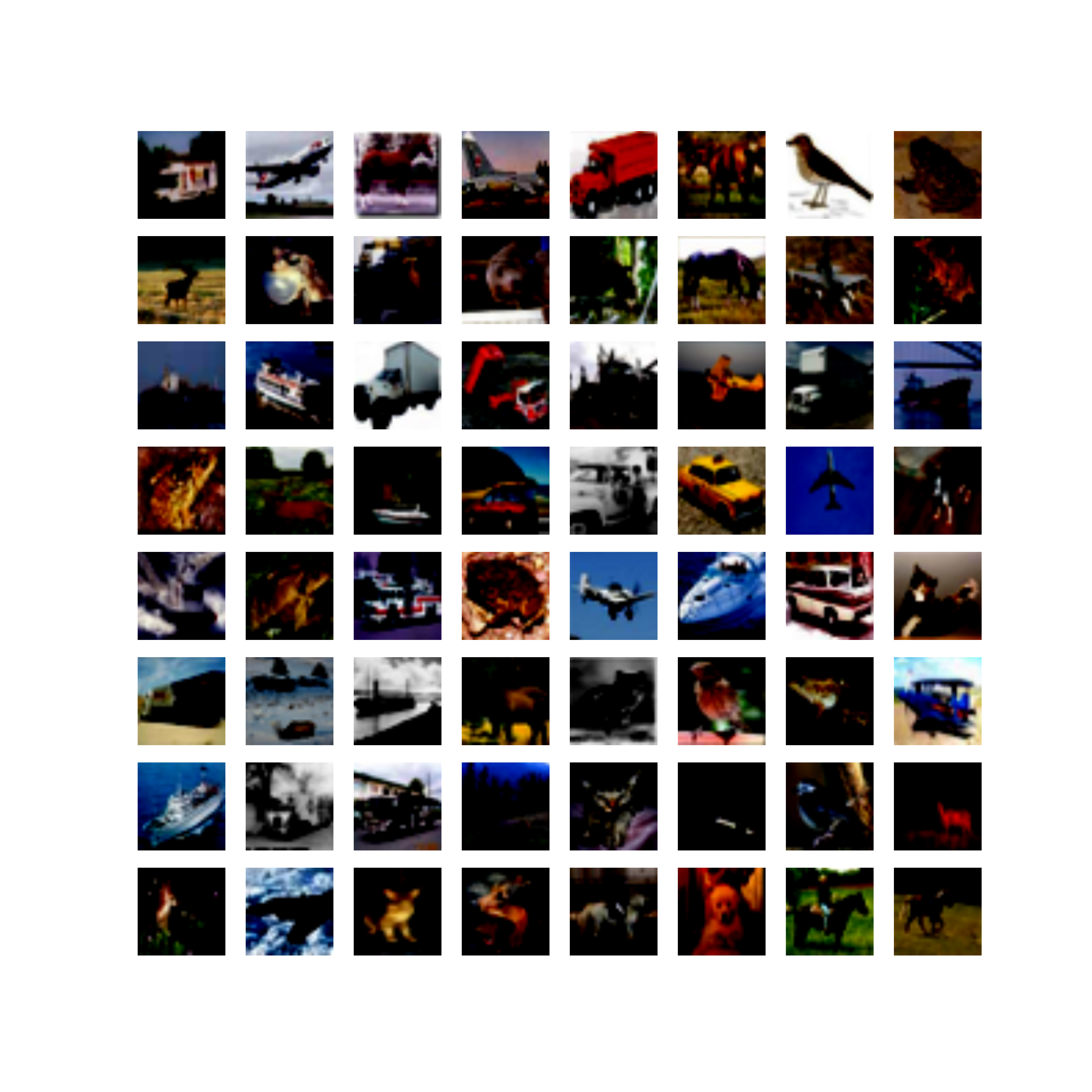}}\label{fig:cifar10_real}    \hspace{-.9cm}
        \subfigure[DCGAN samples
    ]{
    \includegraphics[width=0.52\textwidth]{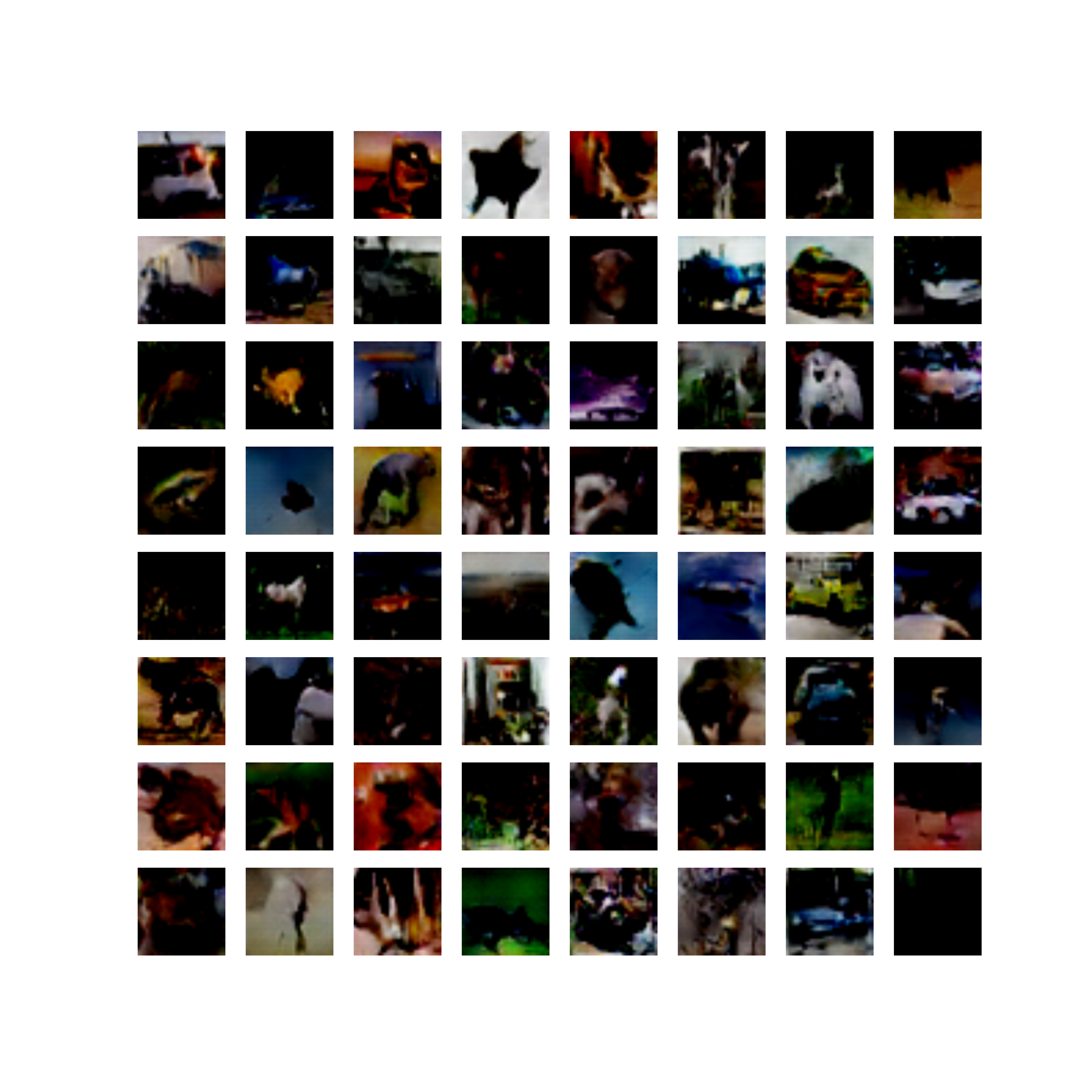}}\label{fig:cifar10_dcgan} 
    
\hspace{-.2cm}        \subfigure[CIFAR10.1 samples
    ]{
    \includegraphics[width=0.52\textwidth]{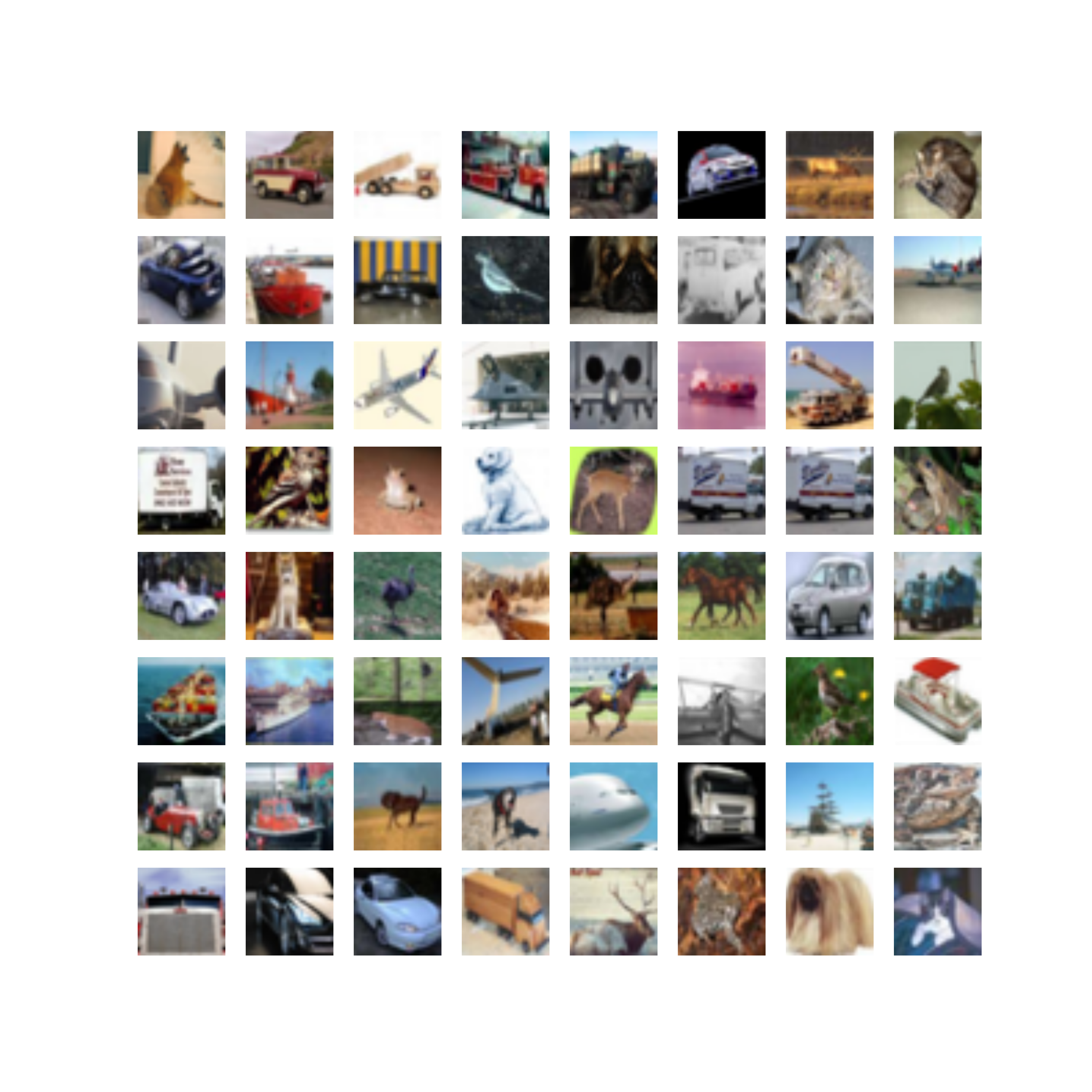}}\label{fig:cifar101}   
\hspace{-.9cm}\subfigure[NCSN samples]{
    \includegraphics[width=0.52\textwidth]{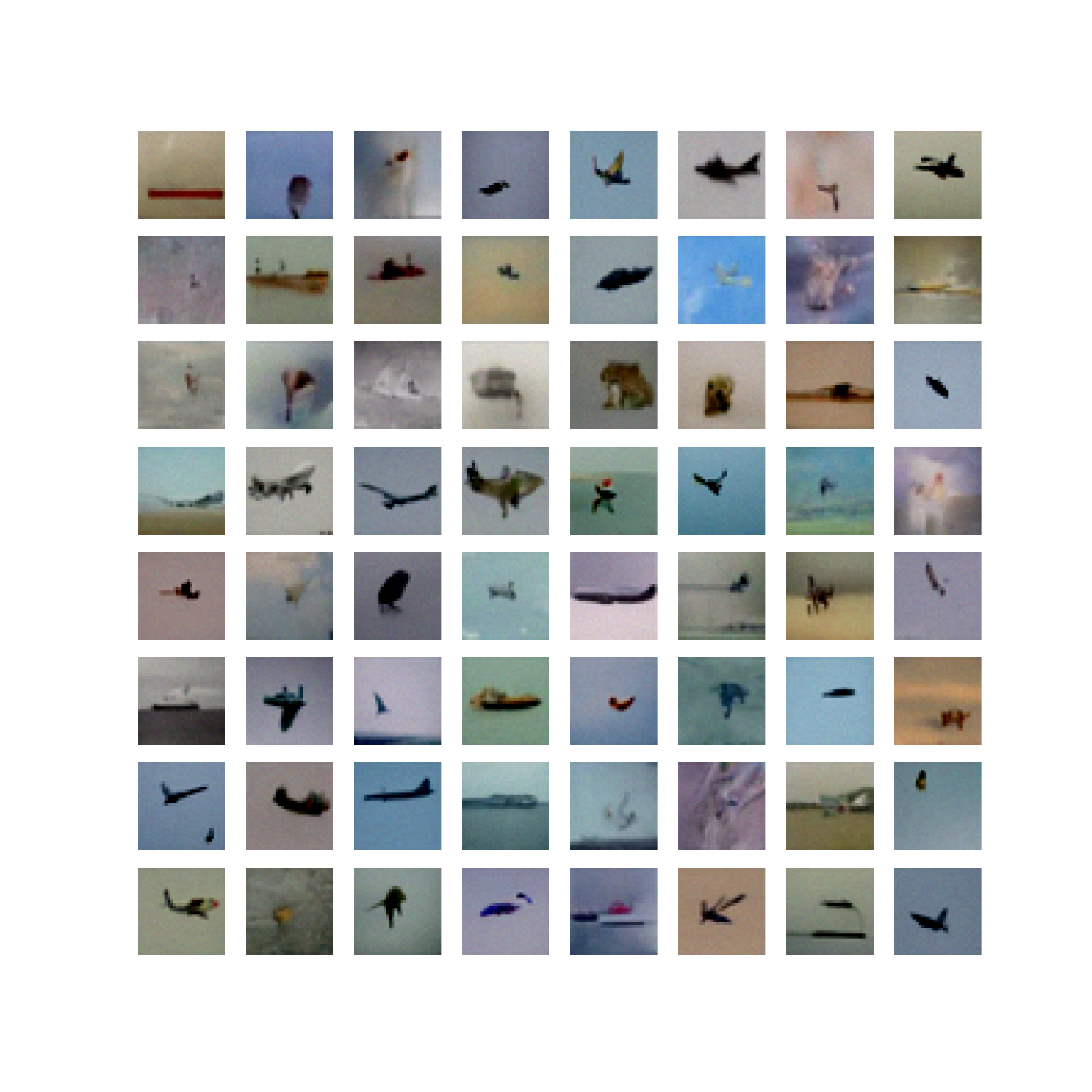}}\label{fig:cifar10_ncsn}

    \caption{CIFAR10 samples
}
   \label{fig:cifar10}
\end{figure}

 \newpage 

\section{Equivalence to the multivariate score-Stein operator}\label{app:three_approaches}

{
Here we show that the operator in Eq.\,\ref{summ2} is equivalent to the corresponding multivariate score-Stein operator in Eq.\,\ref{eq:steinRd}, when they exist;  the difference being the factor $\frac1m$. Recall the set-up for score-Stein operators.
Let {$q$} with smooth support $\Omega_q$ be differentiable. The {score function} of $q$ is the function 
\begin{align*}
  \s_q = \mathcal{T}_{\nabla, q}1 = \grad \log q = \frac{\nabla q}{q}
\end{align*}
(with the convention that $
\s_q \equiv 0$ outside of $\Omega_q$).
The {score-Stein operator} is the vector-valued operator
\begin{equation}
\label{eq:1score}
      \mathcal{A}_q = \nabla   +  
      \s_q \mathrm{I}_{{m}}
    \end{equation}
    acting on differentiable functions $g : \R^{m} \to \R$, {with $\mathrm{I}_{m}$ denoting the $m \times m$ identity matrix}.}
    
    \begin{proposition}\label{prop:equal} 
    When they exist, then the operators in Eq.\,\ref{summ2} and in  Eq.\,\ref{eq:steinRd} differ only by a factor $\frac1m$.
    \end{proposition}
    
    {\bf Proof.}
    Writing $\partial_i$ for the derivative in direction $x^{(i)}$, the score operator acting on differentiable functions $g: \R^m \rightarrow \R$ can be written as 
   \begin{equation}
\label{eq:1score2}
      \mathcal{A}_p g (x) = \sum_{i=1}^m \left\{ \partial_i g(x)  +   g(x) \partial_i (\log q(x)) \right\} .
    \end{equation} 
    Now, for $i \in [m],$
    $$ q(x) = q(x^{(i)} | x^{(j)}, j \ne i) q (  x^{(j)}, j \ne i)$$
    and hence 
    $$ \partial_i (\log q(x))  =
    \partial_i \log q(x^{(i)} | x^{(j)}, j \ne i) .$$ 
    The assertion follows. \hfill $\Box$ 

\begin{example}[Bi-variate Gaussian]\label{ex:bivariate-gauss} Consider $x=(x^{(1)}, x^{(2)})^{\top}\in \R^2$, i.e. $m=2$ and $x\sim \mathcal N(\mu, \Sigma)$ where $\mu=(\mu^{(1)}, \mu^{(2)})^{\top} \in \R^2$, $\Sigma = \begin{pmatrix}
1 & \sigma \\
\sigma & 1
\end{pmatrix}
$. 
With the corresponding precision matrix $\Sigma^{-1}=\frac{1}{1-\sigma^2}\begin{pmatrix}
1 & -\sigma \\
-\sigma & 1
\end{pmatrix}$, it is easy to check $Q^{(1)}(X^{(1)} | X^{(2)} = x^{(2)})\sim \mathcal{N}(\mu^{(1)} + \sigma(x^{(2)} - \mu^{(2)}), 
1-\sigma^2)$. 
For a bi-variate {differentiable} test function $g: \R^2 \to \R$, applying the Stein operator of the form in \cref{opform},
\begin{eqnarray*} 
\A g(x^{(1)}, x^{(2)}) &=& \frac12
\left\{ {\mathcal{T}}^{(1)} g_{x^{(2)
}} (x^{(1)}) + {\mathcal{T}}^{(2)} g_{x^{(1)
}} (x^{(2)}) \right\} \\
&=& \frac12\left\{  (g_{x^{(2)}})'(x^{(1)}) - \frac{x^{(1)}-\mu^{(1)} - \sigma(x^{(2)}-\mu^{(2)})}{1-\sigma^2} g_{x^{(2)}}(x^{(1)}) \right. \\
&& \left. 
+ (g_{x^{(1)}})'(x^{(2)}) - \frac{x^{(2)}-\mu^{(2)} - \sigma(x^{(1)}-\mu^{(1)})}{1-\sigma^2} g_{x^{(1)}}(x^{(2)})
\right\} \\
&=&
\frac12\left\{ \partial_1 g(x^{(1)}, x^{(2)}) - \frac{x^{(1)}-\mu^{(1)} - \sigma(x^{(2)}-\mu^{(2)})}{1-\sigma^2} g_{x^{(2)}}(x^{(1)}) \right. \\
&& \left. 
+ \partial_2 g(x^{(1)}, x^{(2)}) - \frac{x^{(2)}-\mu^{(2)} - \sigma(x^{(1)}-\mu^{(1)})}{1-\sigma^2} g_{x^{(1)}}(x^{(2)})
\right\} \\
&=& \frac12 \left\{ \nabla \times g(x^{(1)}, x^{(2)}) - \Sigma^{-1} ( x^{(1)}-\mu^{(1)}, x^{(2)} - \mu^{(2)})^T g(x^{(1)}, x^{(2)}) \right\} 
\end{eqnarray*} 
where $\partial_i$ denotes the derivative with respect to $x^{(i)}$. Thus, we recover the score operator given in  Eq.\,\ref{eq:steinRd}. 
\end{example} 

\section{Energy-based models and score matching}\label{app:sm_obj}

Energy-based models (EBM{s}) \citep{lecun2006tutorial} have 
been {used} 
in machine learning contexts for modelling and learning deep generative models. An EBM is essentially a Gibbs measure with energy function $E(x)$,
\begin{equation}\label{eq:ebm}
    q(x) = \frac{1}{Z} \exp\{-E(x)\},
\end{equation}
where $Z$ is the (generally) intractable normalisation constant (or partition function). {In particular, learning} 
and training complicated EBM has been studied in machine learning \citep{song2021train}. One of the most popular and relatively stable training objective is the score-matching (SM) objective {given in Eq.\,(\ref{eq:sm_objective})} \citep{hyvarinen2005estimation},
\begin{equation*}
    J(p\|q) = \E_p \left[\left(\log p(x)' - \log q(x)'\right)^2\right],
\end{equation*}
which is particularly useful for the unnormalised models such as EBM{s}.
%

For {an} EBM, {the} SM objective only requires computing $\nabla E(x)$ and $\nabla \cdot \nabla E(x)$ (or $\Delta E(x)$), which is independent of {the} partition function $Z$. 
Moreover, by  learning the SM objective, we can obtain $\nabla \log q(x)$ directly, to construct the approximate Stein operator.

\section{More on kernel-based hypothesis tests}\label{app:mmd}

\subsection{Maximum-mean-discrepancy tests }

\paragraph{Maximum-mean-discrepancy}(MMD) has been introduced as a kernel-based method to tackle two-sample problem{s} \citep{gretton2007kernel}, utilising the rich representation of the functions {in a} reproducing kernel Hilbert space (RKHS) via {a} kernel mean embedding. Let $k: \X \times \X \to \R$ be the kernel associated with RKHS $\H$. 
The kernel \emph{mean embedding} of a distribution $p$ induced by $k$ is defined as 
\begin{equation}\label{eq:mean_embedding}
    \mu_p:= \E_{x\sim p}[k(x,\cdot)] \in \H,
\end{equation}
whenever $\mu_p$ exist. The kernel mean embedding in Eq.\ref{eq:mean_embedding} can be estimated empirically from independent and identically distributed (i.i.d.) samples. Given $x_1,\dots,x_n \sim p$:
\begin{equation}\label{eq:mean_embedding_estimate}
    \widehat{\mu}_p: = \frac{1}{n} \sum_{i=1}^n k(x_i, \cdot)
\end{equation}
replacing $p$ by its empirical counterpart $\widehat{p}=\frac{1}{n} \sum_{i=1}^n \delta_{x_i}$ where $\mathcal{\delta}_{x_i}$ denotes the Dirac measure at $x_i\in \X$. {For i.i.d.\, samples, the}  empirical mean embedding {$\widehat{\mu}_p$} is a $\sqrt{n}$-consistent estimator for $\mu_p$ in RKHS norm 
\citep{tolstikhin2017minimax}, 
{and with $n$ denoting the number of samples,} $\|\mu_p - \widehat{\mu}_p\|_{\H} = O_p(n^{-\frac{1}{2}})$. 
When the sample size $n$ is small, the estimation error {may not be negligible}.

The MMD between two distributions $p$ and $q$ is defined as
\begin{align}
    \operatorname{MMD}(p\| q; \H) 
    &= \sup_{\|{f} \|_{\mathcal{H}} \leq 1} \E_{x\sim p}[f(x)] - \E_{{\tilde x}\sim q}[f({\tilde x})]
    \nonumber\\
    & = \sup_{\|{f} \|_{\mathcal{H}} \leq 1} \left\langle f, \mu_p - \mu_q \right\rangle_{\H} 
     = \| \mu_p - \mu_q  \|_{\H}.\label{eq:mmd_distance}
\end{align}
One desirable property for MMD {is to be able} to distinguish distributions {in the sense that}
$\operatorname{MMD}(p\|q;\H)=0 \Longleftrightarrow p = q$\footnote{Note that MMD is symmetric with respect to $p$, $q$, while KSD is not symmetric with respect to $p$, $q$.}.
{This property} can be achieved via {\it characteristic kernels} \citep{sriperumbudur2011universality}. {It is often more convenient} 
to work with the squared version of MMD:
\begin{align}
    \operatorname{MMD}^2(p\|q;\H) &= \|\mu_p - \mu_q\|^2_{\H} =\left\langle \mu_p, \mu_p \right\rangle + \left\langle \mu_q, \mu_q \right\rangle -2\left\langle \mu_p, \mu_q \right\rangle 
    \nonumber\\
    & = \E_{x, \tilde x \sim p} k(x,\tilde x) +\E_{y, \tilde y \sim q} k(y,\tilde y) - \E_{x\sim p,y\sim q} k(x,y).
\label{eq:mmd2}
\end{align}

Given two sets of 
i.i.d.\,samples 
$\mathbb S_p = \{x_1,\dots,x_n\} \overset{i.i.d.}{\sim} p$ and $\mathbb S_q = \{y_1,\dots,y_l\} \overset{i.i.d.}{\sim} q$,
an unbiased estimator of Eq.\ref{eq:mmd2}, based on the empirical estimate of kernel mean embedding in Eq.\ref{eq:mean_embedding_estimate}, is given by
\begin{align}\label{eq:mmd_u}
{\MMD^2_u}(\mathbb S_p\|\mathbb S_q; \H) 
= \frac{1}{\small n(n-1)}\sum_{i\neq i'}k(x_i,x_{i'}) + \frac{1}{l(l-1)}\sum_{j \neq j'}k(y_{j},y_{j'}) - \frac{2}{nl}\sum_{i j}k(x_i,y_{j}).
\end{align}

{\bf{A two-sample test (or two-sample problem)}} aims to test the null hypothesis $H_0: p = q$ 
against the alternative hypothesis $H_1 : p \neq q$. 
It has been shown that the \textit{asymptotic} distribution of $n$-scaled statistic 
$n\cdot \MMD_u^2(\mathbb S_p\| \mathbb S_q; \H)$ under the null ($p=q$) 
{is that of an} infinite weighted sum of $\chi^2$-distribution \citep[Theorem 12]{gretton2012kernel}, 
{while} under the alternative ($p\neq q$),
the $\sqrt{n}$-scaled statistic
$\sqrt{n}\cdot \MMD_u^2(\mathbb S_p\|\mathbb S_q; \H)$ is asymptotically normally distributed with the mean centered at $\MMD(p\|q; \H)>0$. 
{Thus,} $n\cdot\MMD_u^2(\mathbb S_p\| \mathbb S_q; \H)$ is 
{taken} as a test statistic to be compared against the \emph{rejection threshold}. If the test statistic exceeds the rejection threshold, the empirical estimation of the MMD statistic is thought to exhibit significant departure from the null hypothesis so that $H_0$ is rejected. 

As the null distribution is given by {an} infinite weighted sum of $\chi^2$ random variables which does not have {a} closed form expression, the null distribution can be simulated via a permutation procedure \citep{gretton2008kernel}: Combine and order two sets of samples as $z_i = x_i, i\in[n]$ and $z_{j} = y_{j-n}, n+1\leq j \leq n+l$. {Let} 
$\mu: [n+l]\to [n+l]$ {be a permutation}, and {write} $z^{\mu} = \{z_{\mu_{(1)}},\dots z_{\mu_{(n+l)}}\}$. {Then}  $z^{\mu}$ is re-split into $\mathbb S_p^{\mu} = \{z_i\}_{1\leq i \leq n} $ and $\mathbb S_q^{\mu} = \{z_j\}_{n+1\leq j \leq n+l}$. The permuted MMD is computed via Eq.\,\ref{eq:mmd_u} as
\begin{equation}\label{eq:mmd_perm}
\MMD^2_u(z^{\mu}) = 
{\MMD^2_u}(\mathbb S^{\mu}_p\|\mathbb S^{\mu}_q; \H).    
\end{equation}
For $\mu_1,\dots,\mu_B$, we obtain $\MMD^2_u(z^{\mu_1}),\dots, \MMD^2_u(z^{\mu_B})$ 
{and use these values} 
to compute the empirical quantile of the test statistics $
{\MMD^2_u}(\mathbb S_p\|\mathbb S_q; \H)$.

To {test} whether the implicit generative model can generate samples following the same distribution as the observed sample, it is natural to consider the two-sample problem described above, {which tests whether} 
two sets of samples are 
from the same distribution. In the model assessment context, one set of samples (of size $N$) are generated from the implicit model, while the other set of samples (of size $n$) are observed. 

{The MMD test often assumes that the sample sizes $n$ and $l$ are equal; the asymptotic results including consistency are valid under the regime that $n, l \rightarrow \infty$} 
\citep{gretton2009fast, gretton2012optimal, jitkrittum2016interpretable};
{also the relative model comparisons in  \cite{jitkrittum2018informative} only considered the cases 
$n=l$.}
{In our setting, the sample size $l$ is usually denoted by $N$.} 
For our model assessment problem setting, when $n$ is fixed and $N\to \infty$ is allowed to be asymptotically large, i.e. $n \ll N$, {we find that} the type-I error may not controlled. 
{Hence it is not always the case that} MMD is able to pick up the distributional difference between two sets of samples under the null hypothesis. 
{A} simple experiment in \cref{tab:mmd_typeI_error} shows {an example in which} 
the type-I error is not controlled when $N$ is increasing. {Hence, MMD is \textit{ not} used as comparison for NP-KSD.} 

The high rejection rate of {the} MMD statistic, i.e. {the} high type-I error as $N$ increase, shown in \cref{tab:mmd_typeI_error} can be {
{heuristically explained} as follows.}
Let $\{x_1,\dots,x_n\}, \{\tilde x_1,\dots, \tilde x_N\} \overset{i.i.d.}{\sim} p$ where {the} two sets of samples are generated from the same distribution. {Let} $\widehat \mu_{p,n} = \frac{1}{n}\sum_{i \in [n]} k( x_i,\cdot)$,
and
$\widehat \mu_{p,N} = \frac{1}{N}\sum_{j \in [N]} k(\tilde x_j,\cdot)$. The empirical MMD between $\widehat \mu_{p,n}$ and $\widehat \mu_{p,N}$ can be seen as
\begin{equation}\label{eq:mmd_nN}
    \|\widehat \mu_{p,n} - \widehat \mu_{p,N}\|^2_{\H} =
    \|(\widehat \mu_{p,n} - \mu_p) - (\widehat \mu_{p,N}-\mu_p)\|^2_{\H},  
\end{equation}
where MMD {aims} to {detect} 
the 
{asymptotic equality of} 
$(\widehat \mu_{p,n} - \mu_p)$ and $(\widehat \mu_{p,N} - \mu_p)$. When $n$ is small and fixed, and $n\ll N$, the difference is non-trivial and a rich-enough kernel is able {detect this}
difference, {leading to MMD rejecting the null hypothesis although it is true.} 

\begin{table}[t!]
    \centering
\begin{tabular}{c|rrrrrr}
\toprule

Sample size {N} &  20 &  50 &  100 &  200 &  500 &  1000 \\
\midrule
{MMD} &   0.08 &   0.06 &    \textbf{0.36} &    \textbf{0.9} &    \textbf{1.00} &    \textbf{1.00} \\
{MMDAgg} &   0.06 &   0.07 &    0.02 &    0.03 &    0.02 &     0.05 \\
{KSD} &   0.07 &   0.04 &    0.04 &    0.02 &    0.08 &     0.06 \\
\bottomrule
\end{tabular}
\vspace{0.3cm}
    \caption{Type-I error with increasing sample size $N$. $H_0$ is the standard Gaussian with $m=3$; $n=50$; $\alpha=0.05$; $n_{sim} = 500$, $100$ trials for rejection rate. Bold values show the uncontrolled type-I error.}
    \label{tab:mmd_typeI_error}
    \end{table} 
\paragraph{{MMDAgg, a}  non-asymptotic MMD-based test}
{R}ecently, \cite{schrab2021mmd} proposed an aggregated MMD test that can 
incorporate the setting $n\neq N$ {as long as there exists a constant $C>0$ such that 
{$n \le N \le cn$}. 
Under this condition,}  
MMDAgg is a consistent non-asymptotic test with controlled type-I error, {see \cref{tab:mmd_typeI_error}}, which we use as competitor method in the main text. {In  \cref{tab:mmd_typeI_error} KSD is included as desired method when the underlying null distribution is known.} 

{The}  MMDAgg test statistic is computed by aggregating a set of MMD statistic based on different choices of kernel, e.g. Gaussian bandwidth.
MMDAgg 
{takes into account a number of choices of bandwidth}
$\lambda \in \Lambda$ {where $\Lambda$ is a finite set}. {Let} 
$\widehat{M}_{\lambda}$ {denote} the empirical MMD using {a} kernel with bandwidth $\lambda$.
Each $\lambda$ can be weighted, via $w_{\lambda}$, where $\sum_{\lambda \in \Lambda} w_{\lambda} = 1$. In \citet{schrab2021mmd} (as well as our implemented experiments),  uniform weights {are chosen}; 
$w_\lambda \equiv w = \frac{1}{|\Lambda|}$.
Denote {by} $B_1$  the number of samples {used} to  simulate {the} null distribution for quantile estimation \footnote{{This is the} same as notion $B$ in the main text as well as {in} 
Eq.\,\ref{eq:mmd_perm}.} {and denote by} $B_2$ 
the number of simulated samples {used} to estimate the empirical rejection probability.
Define ${\widehat q}_{\lambda, 1-\alpha}^{B_1}(z^{B_1})$
as the conditional empirical $(1-\alpha)$-quantile {when MMD uses a kernel with bandwidth $\lambda$,}
estimated from the permutation procedure with $B_1$ permutations using Eq.\ref{eq:mmd_perm}.
{Then for a fixed test level $\alpha$, u is estimated via the bi{-s}ection method such that} 
\begin{equation}\label{eq:emp_prob_mmdagg}
    \P\left(\max_{\lambda \in \Lambda} (\widehat M_\lambda - {\widehat q}_{\lambda, uw}^{B_1}(z^{B_1})) >0\right) 
    { \leq \alpha}. 
\end{equation}
{We} reject $H_0$ if for any $\lambda \in \Lambda$ and with the estimated $\widehat u$, we have that $\widehat M_\lambda$ exceeds the rejection probability in Eq.\,\ref{eq:emp_prob_mmdagg}; otherwise {we} do not reject $H_0$.  
In this way, MMDAgg 
{does not only achieve the desired} non-asymptotic type-I error but {is} 
able to explore a wide range of kernels  {in order} 
to produce stronger test power.



\subsection{Wild-bootstrap on KSD testing procedures}

The {wild} bootstrap procedure \citep{chwialkowski2014wild} simulates the null distribution via 
{so-called} wild-bootstrap samples.  
{For KSD,} \cite{chwialkowski2014wild}
has shown 
weak asymptotic convergence to the null distribution with deterministic and bounded kernels. 
For the NP-KSD test statistic, wild-bootstrap samples 
{do} not necessarily converge to the null distribution, due to the estimation difference $(\widehat s^t - s)$, creating a random Stein kernel for NP-KSD. 
{Perhaps therefore unsurprisingly,} the wild bootstrap procedure NP-KSD does not control the type-I error correctly.
Instead, we consider a Monte Carlo procedure to simulate the null distribution {of NP-KSD}. {While Monte Carlo estimation is more comput{ationally} intensive than wild-bootstrap, it is  an accurate method by design.}

\cref{fig:exp_bootstrap} {illustrates this point.  \cref{fig:gaussian_fit} shows samples from a Gaussian distribution. The true density is plotted in red. Two score matching density estimates, SM1 and SM2, are calculated; SM1 presents a good fit whereas SM2 is a less accurate estimate.
For KSD, which is applicable when the underlying null distribution is known, \cref{fig:null_ksd} shows that the Monte Carlo distribution and the wild-bootstrap distribution are close and reach the same conclusion for the KSD test statistic. Using the well-fitting SM1 score density estimate, \cref{fig:null1} gives the Monte Carlo distribution and the wild-bootstrap distribution. The wild-bootstrap distribution is close to the wild-bootstrap distribution for the KSD. In contrast, it 
 differs considerably from the Monte Carlo distribution and would reject the null hypothesis although it is true. \cref{fig:null2} shows a similar result for the not so well fitting estimate SM2. The wild-bootstrap distribution is now more spread out but the observed test statistic is still in the tail of this distribution, whereas it is close to the center of the Monte Carlo distribution. In the {synthetic experiments for MoG in the} main text, 
the model misspecification \textbf{NP-KSD\_G} falls under this setting. These plots illustrate that using wild-bootstrap samples in this scenario could lead to erroneous conclusions. Hence we use Monte Carlo samples.}

\begin{figure}[t!]
    \centering
\subfigure[Samples and fitted densities
    ]{
    \includegraphics[width=0.4612\textwidth]{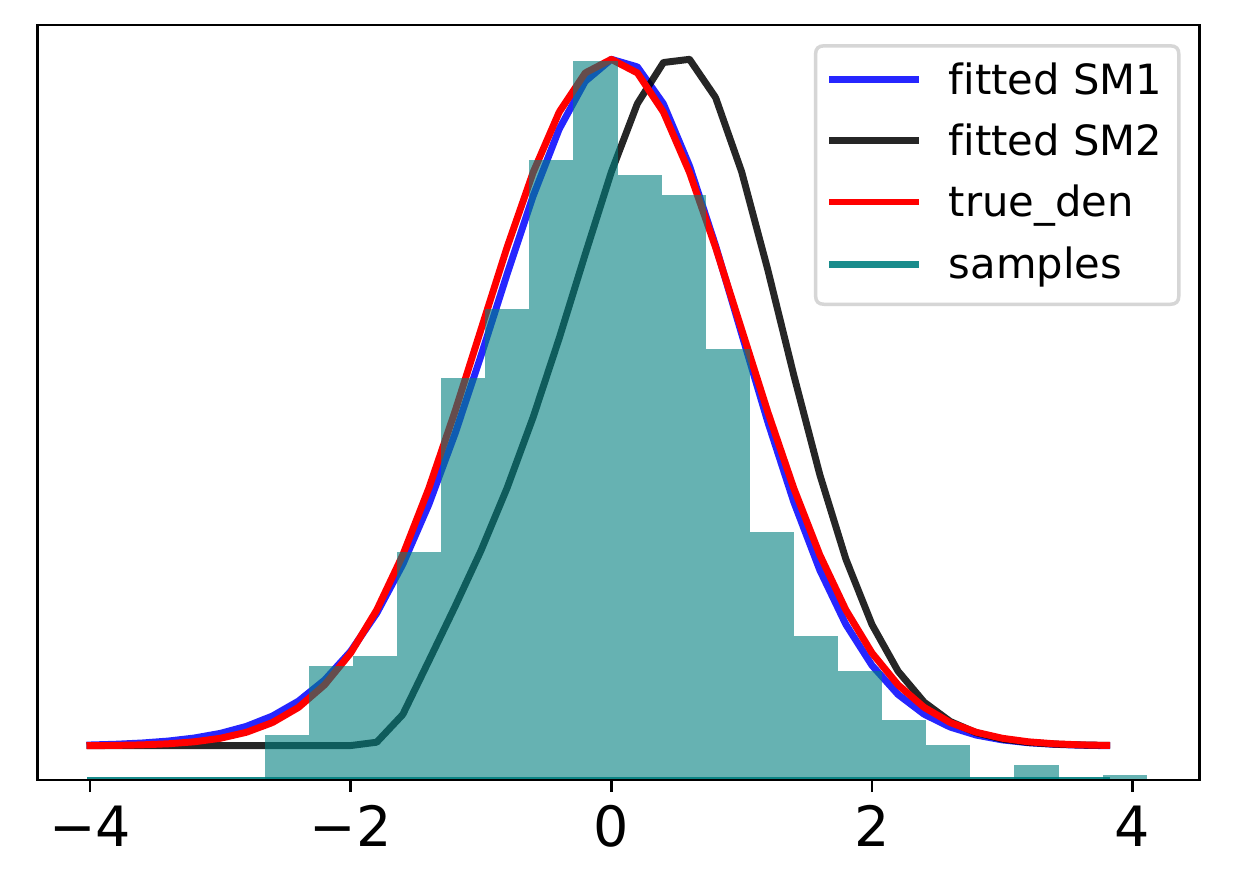}\label{fig:gaussian_fit}}
    
    \vspace{0.5cm}
    \subfigure[Simulated null distributions from KSD
    ]{
    \includegraphics[width=0.46\textwidth]{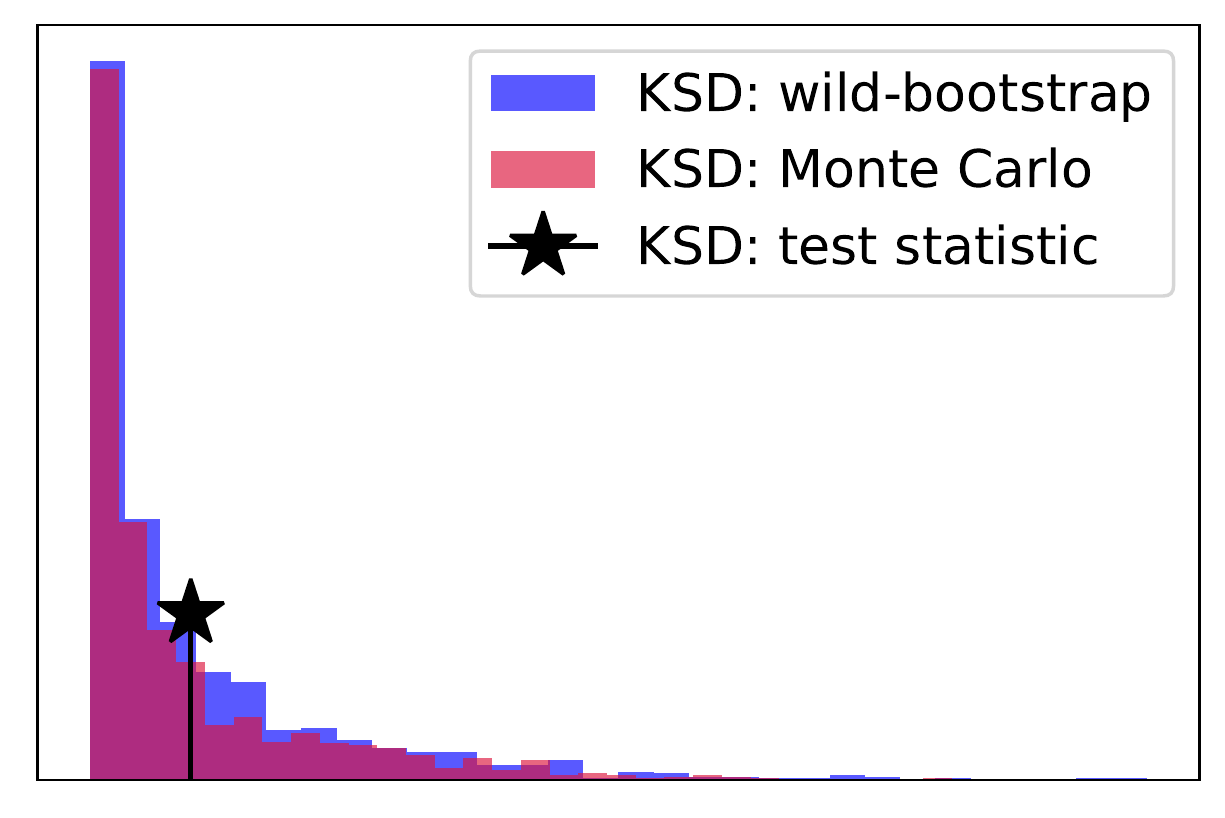}\label{fig:null_ksd}}

    \vspace{0.5cm}
    \subfigure[Simulated null distributions from fitted SM1
    ]{
    \includegraphics[width=0.46\textwidth]{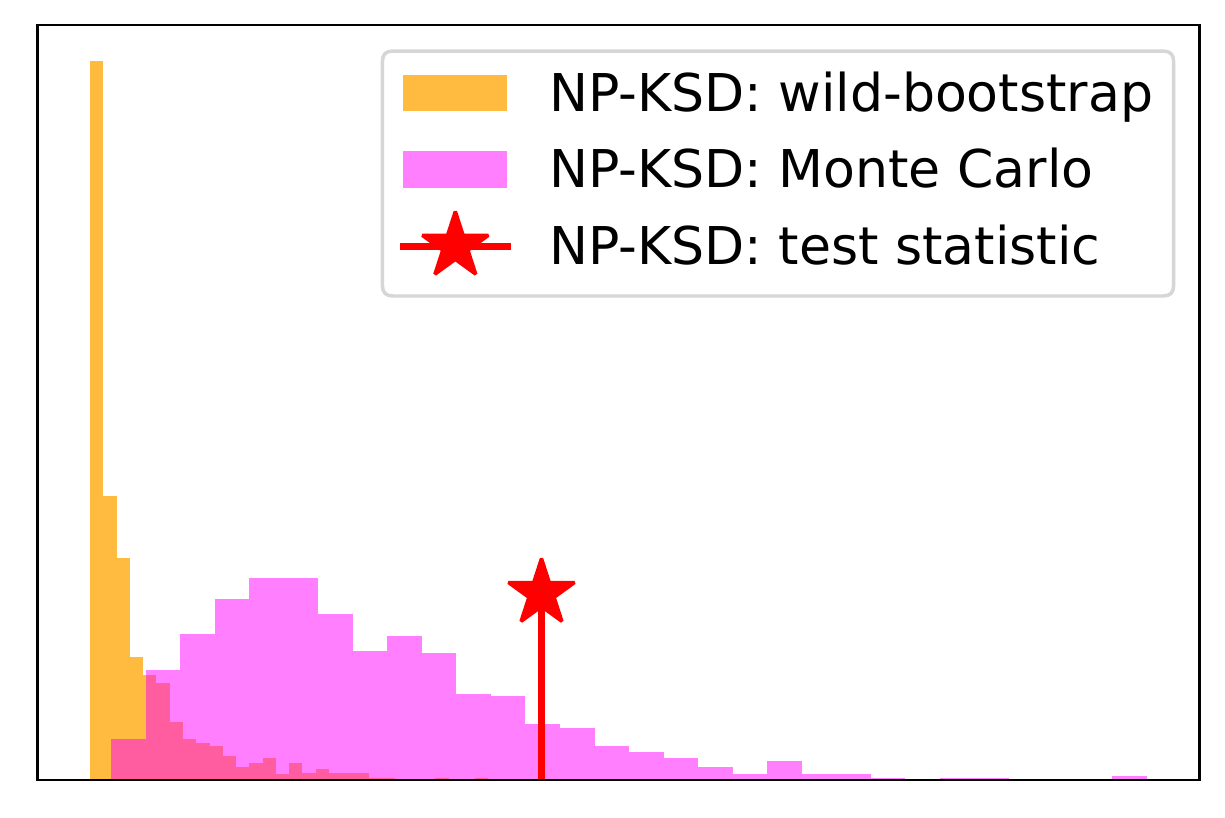}\label{fig:null1}}
        \subfigure[Simulated null distributions from fit SM2
    ]{
    \includegraphics[width=0.46\textwidth]{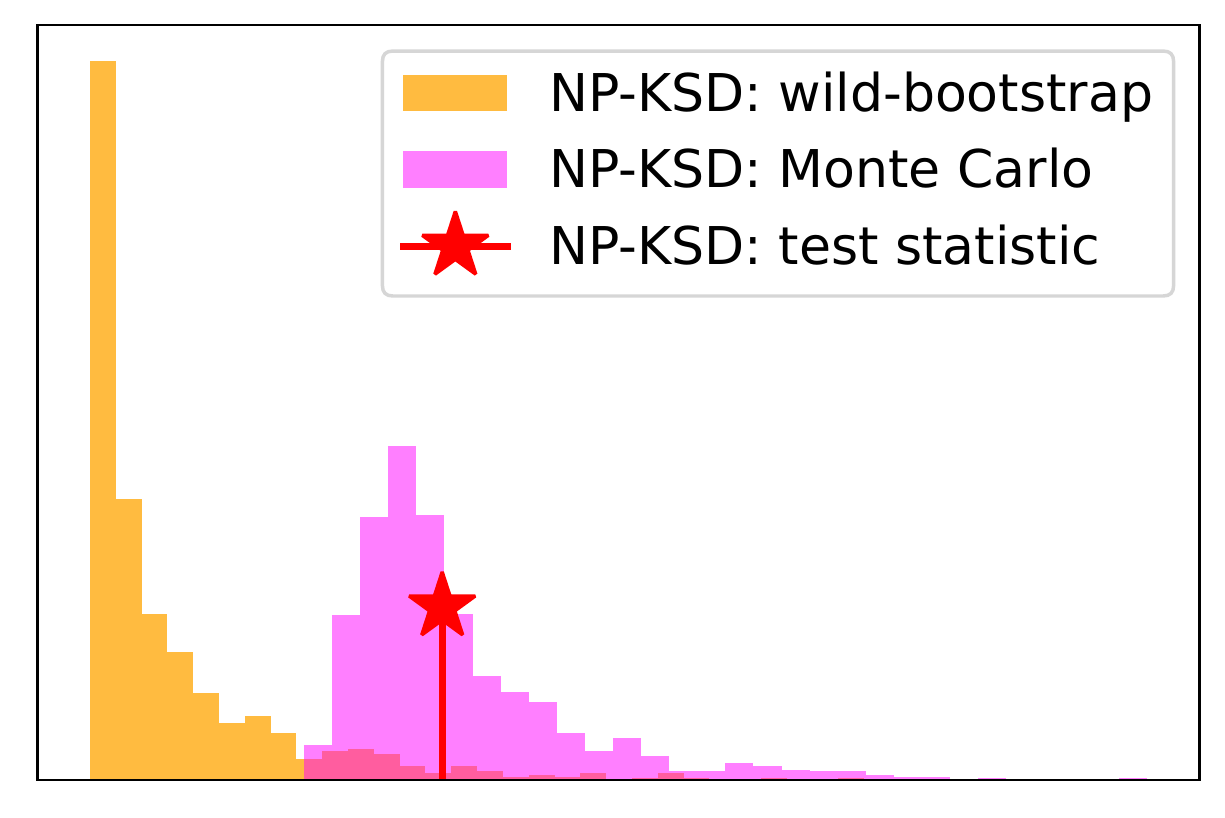}\label{fig:null2}}
    
\vspace{0.5cm}
    \caption{Visualisation for NP-KSD and KSD testing procedures. {For KSD, the wild-bootstrap distribution roughly agrees with the Monte Carlo distribution, whereas for NP-KSD, the wild-bootstrap distribution deviates strongly from the Monte Carlo distribution, indicating a danger of reaching an erroneous conclusion when using wild-bootstrap samples in this scenario.} 
}
    \label{fig:exp_bootstrap}
\end{figure}

\end{document}